\documentclass[DIV=18,fontsize=11pt,abstract]{scrartcl}

\usepackage{amsmath}
\usepackage{amssymb}
\usepackage{mathtools}
\usepackage{amsthm,thmtools,thm-restate,bm}

\usepackage{xcolor}
\usepackage{comment}
\usepackage[english]{babel}
\usepackage[hidelinks]{hyperref}
\usepackage[algoruled,lined,noend,linesnumbered,algo2e]{algorithm2e}
\usepackage[capitalize,nameinlink,noabbrev]{cleveref}
\usepackage{url}
\usepackage[numbers]{natbib}
\usepackage{tcolorbox}
\usepackage{enumitem}
\usepackage{marvosym} 

\usepackage{microtype}
\usepackage{graphicx}
\usepackage{subfigure}
\usepackage{booktabs}

\usepackage{tikz}

\usetikzlibrary{backgrounds} \usetikzlibrary{patterns.meta} \usetikzlibrary{decorations.pathreplacing,calligraphy} \usetikzlibrary{calc} \usetikzlibrary{fadings} 

\tikzfading[name=fade example,right color=transparent!0,left color=transparent!100,middle color=transparent!7]

\newcommand{\indicator}{\mathbb{1}}

\theoremstyle{plain}
\newtheorem{theorem}{Theorem}[section]

\theoremstyle{definition}

\theoremstyle{remark}

\newtheorem{properties}{Properties}

\SetCommentSty{mycommfont}
\makeatletter
\patchcmd{\@algocf@start}{-1.5em}{-2pt}{}{}\let\oldnl\nl \makeatother
\newcommand{\nonl}{\renewcommand{\nl}{\let\nl\oldnl}}

\definecolor{au-blue}{cmyk}{1,0.8,0,0.15}
\definecolor{au-darkblue}{cmyk}{1,0.8,0,0.75}
\definecolor{au-purple}{cmyk}{0.7,0.7,0,0}
\definecolor{au-darkpurple}{cmyk}{0.7,0.7,0,0.75}
\definecolor{au-cyan}{cmyk}{1,0,0,0}
\definecolor{au-darkcyan}{cmyk}{1,0,0,0.75}
\definecolor{au-turkis}{cmyk}{0.8,0,0.4,0}
\definecolor{au-darkturkis}{cmyk}{0.8,0,0.4,0.75}
\definecolor{au-green}{cmyk}{0.6,0,1,0}
\definecolor{au-darkgreen}{cmyk}{0.6,0,1,0.75}
\definecolor{au-yellow}{cmyk}{0,0.3,1,0}
\definecolor{au-darkyellow}{cmyk}{0,0.3,1,0.75}
\definecolor{au-orange}{cmyk}{0,0.6,1,0}
\definecolor{au-darkorange}{cmyk}{0,0.6,1,0.75}
\definecolor{au-red}{cmyk}{0,1,1,0}
\definecolor{au-darkred}{cmyk}{0,1,1,0.75}
\definecolor{au-magenta}{cmyk}{0,1,0,0}
\definecolor{au-darkmagenta}{cmyk}{0,1,0,0.75}
\definecolor{au-grey}{cmyk}{0,0,0,0.6}
\definecolor{au-darkgrey}{cmyk}{0,0,0,0.85}

\newcommand{\cce}{\zeta}

\newcommand{\cmu}{8\alpha^2}

\newcommand{\cmf}{c_3}

\newcommand{\cg}{\alpha}

\newcommand{\ck}{\tilde{c}_1}

\newcommand{\cpt}{\tilde{c}_2}
\newcommand{\cptt}{\tilde{c}_3}

\newcommand{\cmo}{c_0} \newcommand{\cme}{c_1} \newcommand{\cmt}{c_2} 
\newcommand{\samp}{\mathbf{S}}

\newcommand{\cre}{\tilde{\alpha}}
\newcommand{\crt}{\alpha'}

\newcommand{\Ss}{\mathit{supp}(D)}
\newcommand{\eps}{\varepsilon}
\newcommand{\E}{\mathbb{E}}
\newcommand{\V}{\textrm{Var}}
\newcommand{\Hyp}{\mathcal{H}}
\newcommand{\W}{\mathcal{W}}
\newcommand{\pr}{\mathbb{P}}
\newcommand{\Alg}{\mathcal{A}}
\newcommand{\Loss}{\mathcal{L}}
\newcommand{\Concepts}{\mathcal{C}}
\newcommand{\Uni}{\mathcal{U}}
\renewcommand{\log}{\lg}

\newcommand{\Xs}{\mathcal{X}}
\newcommand{\Dist}{\mathcal{D}}
\newcommand{\dist}{\mathcal{D}}
\newcommand{\R}{\mathbb{R}}
\newcommand{\Ws}{\mathcal{W}}
\DeclareMathOperator{\sign}{sign}

\newcommand{\Sss}{\mathit{supp}(\dist)}
\newcommand{\distc}{\Delta_{\Xs}}
\newcommand{\aW}{\mathcal{W}_{\mathcal{H}}}
\newcommand{\aWs}{\mathcal{W}_{\mathbf{H}}}
\newcommand{\sh}{\mathbf{H}}
\newcommand{\thh}{t_{\mathbf{H}^2}}
\newcommand{\gh}{g_{\mathbf{H}^1}}
\newcommand{\eg}{E_{S}}
\newcommand{\ege}{E_{S}^1}
\newcommand{\egt}{E^2}
\newcommand{\he}{\mathbf{H}^1}
\newcommand{\htt}{\mathbf{H}^2}
\newcommand{\hi}{\mathbf{H}_i}
\newcommand{\spo}{\mathcal{S}_\text{part1}}
\newcommand{\wa}{w^{\Alg}}

\newcommand{\shh}{\mathbf{h}}
\newcommand{\frs}{F_{r,S}}

\newcommand{\brackets}[1]{\left(#1\right)}
\newcommand{\braces}[1]{\left\{#1\right\}}

\newcommand{\bool}{\{-1,1\}}

\AtBeginDocument{\renewcommand{\epsilon}{\varepsilon}}

\newcommand{\alert}[1]{\textbf{\color{green}
[#1]}\marginpar{\textbf{\color{green}**}}\typeout{ALERT:
\the\inputlineno: #1}}

\newcommand{\mrinline}[1]{\todo[color=au-orange!30,inline]{\scriptsize\textbf{MR:} #1}}
\newcommand{\mrrand}[1]{\todo[color=au-orange!30]{\scriptsize\textbf{MR:} #1}}

\usepackage[disable,textsize=tiny]{todonotes}

\usepackage[bb=dsserif]{mathalpha} 

\setcitestyle{authoryear,round,citesep={;},aysep={,},yysep={;}}

\newcommand{\email}[1]{\texttt{\href{mailto:#1}{#1}}}

\title{AdaBoost is not an Optimal Weak to Strong Learner}
\author{Mikael Møller Høgsgaard, Kasper Green Larsen, Martin Ritzert\\
{\small \email{hogsgaard@cs.au.dk}, \email{larsen@cs.au.dk}, \email{ritzert@informatik.uni-goettingen.de}}}

\begin{document}

\date{}
\maketitle

\begin{abstract}
AdaBoost is a classic boosting algorithm for combining multiple inaccurate
classifiers produced by a \emph{weak learner}, to produce a \emph{strong
learner} with arbitrarily high accuracy when given enough training data. 
Determining the optimal number of samples necessary to obtain a given
accuracy of the strong learner, is a basic learning theoretic
question. Larsen and Ritzert (NeurIPS'22) recently presented the first
provably optimal weak-to-strong learner. However, their algorithm is
somewhat complicated and it remains an intriguing question whether the prototypical
boosting algorithm AdaBoost also makes optimal use of training samples.
In this work, we answer this question in the negative. Concretely, we
show that the sample complexity of AdaBoost, and other classic variations thereof, are sub-optimal by at least
one logarithmic factor in the desired accuracy of the strong learner.
\end{abstract}

\section{Introduction}
\label{sec:intro}
The algorithm AdaBoost \citep{adaboost} is the textbook example of a boosting algorithm. Boosting algorithms in general make use of a \emph{weak learner}, i.e. a learning algorithm that produces classifiers with accuracy slightly better than chance, and produces from it a so-called \emph{strong learner}, achieving arbitrarily high accuracy when given enough training samples. 
The question whether one can always produce a strong learner from a weak learner was initially asked by Kearns and Valiant \cite{kearns1988learning,kearns1994cryptographic} and initiated the field of boosting.

Given a weak learner $\W$, AdaBoost uses $\W$ to train multiple inaccurate classifiers/hypotheses that focus on different parts of the training data and combines them using a weighted majority vote. In more detail, it runs for some $T$ iterations, each time invoking $\W$ to produce a hypothesis $h_t$. It then computes weights $w$ and outputs the final voting classifier $f(x) = \sign(\sum_t w_t h_t(x))$. 
For the calls of $\W$, AdaBoost maintains a distribution $\dist_t$ over the training samples that puts a large weight on training samples misclassified by most of $h_1,\dots,h_{t-1}$ and a smaller weight on samples classified correctly. 
Using this distribution, in iteration $t$ AdaBoost invokes the weak learner to produce a hypothesis $h_t$ performing better than chance under $\dist_t$.
This way, $h_t$ focuses on training examples which are hard for the voting classifier so far.

In this paper, we study the sample complexity of AdaBoost, answering the question whether AdaBoost is able to make optimal use of its training data. To formally answer this question, we need to introduce a few parameters. A $\gamma$-weak learner is a learning algorithm that, given some constant number of training samples from an unknown data distribution $\dist$, produces a hypothesis $h$ that correctly predicts the label of a new sample from $\dist$ with probability at least $1/2+\gamma$. We let $\Hyp$ denote the set of possible hypotheses that the weak learner may output. A strong learner, on the other hand, is a learning algorithm that for any $0 < \eps,\delta < 1$, with probability at least $1-\delta$ over a set of $m(\eps,\delta)$ training samples from an unknown distribution $\dist$, outputs a hypothesis that correctly predicts the label of a new sample from $\dist$ with probability at least $1-\eps$. The function $m(\eps,\delta)$ is referred to as the sample complexity. A strong learner can thus obtain arbitrarily high accuracy $1-\eps$ when given enough training samples $m(\eps,\delta)$. See \cref{sec:prelim} for a formal definition of weak and strong learning.

Recently, \citet{optimalWeakToStrong} showed that the optimal sample complexity of weak-to-strong learning is given by
\begin{align}
  m(\epsilon,\delta) = \Theta \brackets{\frac{d}{\gamma^2\epsilon} + \frac{\ln (1/\delta)}{\epsilon}},
  \label{eq:optimalSampleComplexity}
\end{align}
where $d$ is the VC-dimension of the hypothesis set $\Hyp$ of the weak learner. 
The paper provides both a learning algorithm achieving this sample complexity as well as an asymptotically matching lower bound. Their algorithm is based on a majority vote among hypotheses produced by a version of AdaBoost. It is thus a majority of majorities. Is this necessary for optimal weak-to-strong learning? Or does it suffice to use a classic algorithm like AdaBoost? The current best upper bound on the sample complexity of AdaBoost (for constant $\delta$) is~\cite{understandingMachineLearning}: 
\begin{align}
    \label{eq:ada}
    m_\text{Ada}(\eps) = O\brackets{\frac{d \ln \frac{1}{\epsilon \gamma} \ln \frac{d}{\epsilon \gamma}}{\gamma^2 \epsilon}}
\end{align}
However, this is just an upper bound, and until now, it remained completely plausible that a better analysis could remove the two logarithmic factors.

The main contribution of this work is to show that AdaBoost is \emph{not always optimal}. Concretely, we show that there exists a weak learner $\W$, such that if AdaBoost is run with $\W$ as its weak learner, its sample complexity is sub-optimal by at least one logarithmic factor. This is stated in the following theorem:
\begin{theorem}
  \label{thm:mainTheorem}
  For any $0 < \gamma < C$ for $C>0$ sufficiently small, any $d = \Omega(\ln(1/\gamma))$, and any $\exp(-\exp(\Omega(d))) \leq \eps \leq C$, there exists a $\gamma$-weak learner $\W$ using a hypothesis set $\Hyp$ of VC-dimension $d$ and a distribution $\dist$, such that AdaBoost run with $\W$ is sub-optimal and needs 
\[
    m_\text{Ada}(\epsilon) = \Omega\brackets{\frac{d \ln (1/\epsilon)}{\gamma^2 \eps}}
  \]
   samples from $\dist$ to output with constant probability, a hypothesis with error at most $\eps$ under $\dist$.
\end{theorem}

This lower bound does not only apply to AdaBoost but extends to many of its variants such as AdaBoost$_\nu$ \cite{ratsch2002maximizing}, AdaBoost$^*_\nu$ \cite{ratsch2005efficient}, and DualLPboost \cite{grove1998boosting}.
The key property those algorithms share and that we manage to exploit is that they run the weak learner $\W$ on the full training data set.
This allows $\W$ to adversarially return hypotheses that accumulate mistakes outside of the training data, leading to poor generalization performance.

The rest of the paper is structured as follows.
In the remainder of this section, we describe some preliminaries and give an overview of the related work. In \cref{sec:overview}, we present the high-level ideas of our proof and in \cref{sec:maintheorem} we sketch the formal details of the proof.
The proofs of the main lemmas and parts of the formal proof of \cref{thm:mainTheorem} are deferred to the appendix.

\subsection{Preliminaries and Notation}
\label{sec:prelim}
We now formally define our setup.
Weak and strong learning are studied in the general framework of \emph{probably approximately correct} (PAC) learning, see e.g. \cite{understandingMachineLearning} for an introduction.
In the PAC learning framework, one assumes that training samples are chosen i.i.d. from an underlying distribution $\dist$ over elements of some universe $\Xs$. 
Furthermore, we assume an underlying but unknown `correct' labeling function $c\colon \Xs \rightarrow \bool$ called the \emph{concept}, which assigns every element from the universe $\Xs$ its `true' label. The concept is assumed to belong to a concept class $\Concepts \subseteq \Xs \to \bool$.

A learning algorithm $\Alg$ is a $\gamma$-\emph{weak learner} for $\Concepts$, if for every distribution $\dist$ over $\Xs$ and every concept $c \in \Concepts$, there is a constant number of samples $m_0$ and a constant $\delta_0 < 1$, such that with probability at least $1-\delta_0$ over $m_0$ i.i.d. samples $x_1,\dots,x_{m_0}$ from $\dist$ and their corresponding labels $c(x_1),\dots,c(x_{m_0})$, $\Alg$ outputs a hypothesis $h$ with error
\[
    \Loss_{\dist}(h) = \Pr_{x \sim \dist}[h(x) \neq c(x)] \leq 1/2-\gamma.
\]
We refer to $\gamma$ as the \emph{advantage} of the weak learner. We let $\Hyp$ denote the hypothesis set used by the weak learner, i.e. we assume that $h \in \Hyp$ and that $\Hyp$ has a finite VC-dimension~$d$.

A learning algorithm $\Alg$ is a \emph{strong learner} for $\Concepts$, if for every $0 < \eps, \delta < 1$, there exists some number of samples $m(\eps,\delta)$, such that with probability at least $1-\delta$ over $m(\eps,\delta)$ i.i.d. samples from $\dist$ and their corresponding labels, $\Alg$ outputs a hypothesis with error
\(
    \Loss_{\dist}(h) \leq \eps.
\)

AdaBoost is the classic algorithm for constructing a strong learner from a $\gamma$-weak learner. 
For completeness, we have included the full algorithm as \cref{alg:adaboost}.

\begin{algorithm2e}[t]
  \DontPrintSemicolon
  \KwIn{training set $S = \{(x_1,c(x_1)),\dots,(x_m,c(x_m))\}$,\\
    \quad number of rounds $T$\\
}
  \KwResult{A majority hypothesis $h_\text{out}$}
  \SetKw{Break}{break}

  $ \dist^{(1)} \gets \left( \frac{1}{m},\dots \frac{1}{m} \right)$ \tcp*{uniform init of $\dist$}

  \For(){$t= 1,\dots,T$}{
    $h_t \gets \W(\dist^{(t)},S)$ \tcp*{invoke weak learner $\W$}
    
    $\gamma_t \gets \sum_{i=1}^m \dist^{(t)} \operatorname{sign}(c(x_i) h_t(x_i))$ \tcp*{ error}
    
    $w_t = \frac{1}{2} \ln \brackets{\frac{1-\gamma_t}{\gamma_t}}$
    \tcp*{weight for $h_t$}
    
    \For(){$i \in \{1,\dots,m\}$}{
        \tcc{update $\dist$ based on success of $h_t$}
        $ \dist^{(t+1)}_i \gets 
            \frac{\dist^{(t)}_i \exp\big(-w_t c(x_i) h_t(x_i)\big)}{\sum_{j=1}^{m} \dist^{(t)}_j\exp\big(-w_t c(x_j) h_t(x_j)\big) }$ }
    
  }

  \Return{$h_\text{out}(x) = \operatorname{sign} \brackets{ \sum_{t=1}^T w_t h_t(x)}$}

  \caption{AdaBoost}\label{alg:adaboost}
\end{algorithm2e}

\subsection*{Related Work}
In terms of sample complexity, most previous works prove generalization bounds for \emph{voting classifiers} in general. A voting classifier over a hypothesis set $\Hyp$, is a majority vote $f(x) = \sign\brackets{\sum_{h \in \Hyp} \alpha_h h(x)}$ for coefficients $\alpha_h > 0$ such that $\sum_h \alpha_h = 1$. AdaBoost can be seen to output a voting classifier by appropriate normalization of the coefficients $w_t$ chosen in \cref{alg:adaboost}. The generalization bounds for voting classifiers are typically \emph{data-dependent} in the sense that they depend on the so-called \emph{margin} of the voting classifier. For a voting classifier $f(x) = \sign\brackets{\sum_{h \in \Hyp} \alpha_h h(x)}$ and a sample $(x,c(x))$, the margin of $f$ on $(x,c(x))$ is defined as $c(x) \sum_{h \in \Hyp} \alpha_h h(x)$. The margin is thus a number between $-1$ and $1$ and is positive if and only if $f(x) = c(x)$. Intuitively, large margins correspond to high certainty/agreement among the hypotheses. 
In terms of upper bounds, Breiman~\cite{breiman1999prediction} showed that with probability $1-\delta$ over a training set $S$ of $m$ samples, all voting classifiers $f$ with margin at least $\gamma$ on all samples in $S$ have 
\begin{align}
    \label{eq:breiman}
    \Loss_{\dist}(f) = O\left(\frac{d \ln(m/d) \ln m}{\gamma^2 m}\right).
\end{align}
A small tweak to AdaBoost, known as AdaBoost$^*_{\nu}$~\cite{ratsch2005efficient}, guarantees that the output hypothesis $f$  has margins $\Omega(\gamma)$ on all samples when AdaBoost$^*_\nu$ is run with a $\gamma$-weak learner. Solving for $\eps = \Loss_{\dist}(f)$ in \cref{eq:breiman} matches the sample complexity bound for AdaBoost from \cref{eq:ada}.

In terms of sample complexity lower bounds for boosting, or for AdaBoost in particular, there are some relevant works. First, as mentioned earlier and stated in \eqref{eq:optimalSampleComplexity}, it is known that any weak-to-strong learner must have a sample complexity of $\Omega\big(d/(\gamma^2 \eps) + \ln(1/\delta)/\eps\big)$~\cite{optimalWeakToStrong}. 
While not directly comparable, work by \cite{boostingLowerBound} showed that there are data distributions, such that with constant probability over a set of $m = (d/\gamma^2)^{1 + \Omega(1)}$ samples, there \emph{exists} a voting classifier $f$ with margin at least $\gamma$ on all samples, yet its generalization error is at least $\Omega\big(d\ln(m)/(\gamma^2 m)\big)$. This lower bound is in some sense similar to our work, as it manages to squeeze out a logarithmic factor. However, the voting classifier $f$ is only shown to exist and as such might not correspond to the output of any reasonable learning algorithm, certainly not AdaBoost. 

At this point, we would like to compare AdaBoost to the optimal weak-to-strong learning algorithm given by \citet{optimalWeakToStrong}.
First, their learning algorithm is more complicated. It runs AdaBoost$^*_\nu$ on various sub-samples of the training data to obtain voting classifiers $f_1,\dots,f_T$ which it then combines in a majority vote $g(x) = \sign(\sum_i f_i(x))$.
It thus outputs a majority of majorities.
Moreover, the number of sub-samples is a rather large $T = m^{\lg_4 3} \approx m^{0.79}$ and their size is linear in the overall number of training samples $m$, thus resulting in somewhat slow training time. 
The sub-samples are constructed with a very careful overlap as pioneered by \citet{hanneke2016optimal} in his optimal algorithm for PAC learning in the realizable setting. 
A recent manuscript~\cite{baggingIsOptimal} shows that one may replace the $T$ sub-samples by just $O(\lg(m/\delta))$ bootstrap samples (sub-samples each consisting of $m$ samples with replacement from the training data) in the algorithm from \citet{optimalWeakToStrong}. 
While reducing the number of sub-samples, it still remains a majority of majorities.
It would thus have been desirable if one could show that AdaBoost also had an optimal sample complexity.
Sadly, as already stated in \cref{thm:mainTheorem}, this is not true. \section{Proof Overview}
\label{sec:overview}
In this section, we give an overview of the main ideas in our proof that AdaBoost is not always an optimal weak-to-strong learner. Concretely, for any $\gamma$, $m$, and $d = \Omega(\ln(1/\gamma))$ we show that there exists an input domain $\Xs$, a distribution $\dist$ over $\Xs$, a concept $c : \Xs \to \{-1,1\}$, a hypothesis set $\Hyp$ of VC-dimension at most $d$, and a $\gamma$-weak learner $\W$ for $c$ that outputs hypotheses from $\Hyp$, such that with constant probability over a set of $m$ samples $S \sim \dist^m$ and their corresponding labels $c(S)$, AdaBoost run with the weak learner $\W$ produces a voting classifier $f$ with $\Loss_{\dist}(f) = \Omega\big(d \ln(\gamma^2 m/d)/(\gamma^2 m)\big)$. Solving for $\eps = \Loss_{\dist}(f)$ gives $m = \Omega\big((d \ln(1/\eps))/(\gamma^2 \eps)\big)$ as claimed in \cref{thm:mainTheorem}.

When proving the lower bound for AdaBoost, we consider just one fixed concept $c$, namely the concept that assigns the label $1$ to all elements of $\Xs$. AdaBoost of course does not know this but executes precisely as in \cref{alg:adaboost}. 
As distribution $\dist$ we consider the uniform distribution $\Uni$ over the input domain $\Xs$. 
Thus, if $f$ is the output of AdaBoost and $\Xs = [u]$, then $\Loss_{\Uni}(f)$ is precisely equal to the fraction of elements  $i \in [u]$ for which $f(i)=-1$. Our goal is thus to show that AdaBoost will produce a voting classifier $f$ with a negative prediction on many $i \in [u]$.

To prove the above, we need to construct a weak learner $\W$ that somehow returns hypotheses that result in AdaBoost making many negative predictions. 
Although the formal definition of a $\gamma$-weak learner given in \cref{sec:prelim} allows $\W$ to sometimes (with probability $\delta_0$) return a hypothesis with advantage less than $\gamma$, we will \emph{not} do so in our construction.
Thus, our adversarial weak learner always returns hypotheses with advantage at least $\gamma$ which only makes our lower bound stronger.

To define our adversarial weak learner $\W$, we carefully examine the ``interface'' it must support. Concretely, the way AdaBoost accesses a weak learner is to feed it the training data $S=\{(x_i,c(x_i))\}_{i=1}^m$ and a distribution $\dist_t$ over $S$. From this, AdaBoost expects that $\W$ returns a hypothesis $h_t$ with advantage at least $\gamma$ under the distribution $\dist_t$ which is supported only on $S$. Our adversarial weak learner $\W$ will support this interface. In fact, it will completely ignore the set $S$ and return a hypothesis that is solely a function of $\dist_t$.
Our weak learner thus needs to be a function, that for any probability distribution $\dist$ over $\Xs$ returns a hypothesis $h$ with advantage at least $\gamma$ under $\dist$ (for the all-1 concept~$c$).\mrrand{here we have that switch between achieving $\gamma$ on $S$ and on $\Xs$.}

Our main challenge is now to design a weak learner that always has advantage $\gamma$ under the distributions fed to it by AdaBoost, yet under the uniform distribution $\Uni$ over $\Xs = [u]$, the voting classifier produced by AdaBoost must often make negative predictions. 
Here, our first observation is that if the universe size $u$ is $c m/\ln(\gamma^2 m/d)$ for a sufficiently small constant $c>0$, then by a coupon collector argument, with constant probability there are $\Omega(d/\gamma^2)$ elements $i \in [u]$ that are not sampled into the training set $S$.
Our basic idea is to force that the final voting classifier $f$ produced by AdaBoost makes negative predictions on a constant fraction of these non-sampled elements. 
This would imply $\Loss_{\Uni}(f) = \Omega\big((d/\gamma^2)/u\big) = \Omega\big(d \ln(\gamma^2 m/d)/(\gamma^2 m)\big)$ as claimed.

Our next key observation is that all distributions $\dist_t$ fed to $\W$ by AdaBoost  put a non-zero probability on \emph{every} element in the training data set. Crucially, this implies that the weak learner knows the complete training set and can thus compute the $\Omega(d/\gamma^2)$ points $\bar{S}$ that were not sampled. \mrrand{It feels like an interesting sidenote that the WL can always compute $S$ on its own}
Our adversarial weak learner does precisely this and chooses an arbitrary subset $\bar{S}' \subseteq \bar S$ of size $O(d/\gamma^2)$ (the same deterministic choice for a given $\bar{S}$). 
It then returns a hypothesis $h$ that has advantage $\gamma$ under $\dist_t$ but at the same time under the uniform distribution over $\bar{S}'$ is \emph{wrong} with probability $1/2 + \gamma$. Notice that it is wrong on $\bar{S}'$ with probability more than half which we call a \emph{negative advantage} of $-\gamma$.
Intuitively, since this holds for every $h$ returned by $\W$ on an execution of AdaBoost (for the same $\bar{S}'$), the output $f$ of AdaBoost will be mistaken on about half the points in $\bar{S}'$ which is sufficient for the lower bound.

To carry out the above argument, we need to construct a hypothesis set $\Hyp$ that contains hypotheses with advantage $\gamma$ on $S$ under $D_t$ and negative advantage over $\bar S'$.
Then the weak learner can essentially just return such a hypothesis.
For this construction, we use a probabilistic argument and show that by sampling a \emph{random} hypothesis set $\Hyp$ in an appropriate manner and defining an associated weak learner $\W_{\Hyp}$, there is a constant probability that the weak learner satisfies all of the above. Hence, a weak learner must exist. The point of considering a random $\Hyp$ is that it allows us to give simple probabilistic arguments that show that all the hypotheses that $\W_{\Hyp}$ needs to return on an execution of AdaBoost indeed exist in $\Hyp$. We illustrate $\Hyp$ in \cref{fig:hypothesis}.
\begin{figure}
    \begin{center}
    \begin{tikzpicture}[y=0.6cm,
]
\useasboundingbox (-1.8,0.1) rectangle (5.1,5.4);
  \fill[fill=black!20] (0,4.6) rectangle ++(5,.8);
  \node (universe) at (2.5,5) {$1,2,3,4,5,\dots,u-2,u-1,u$};
\node (X) at (-.4,5) {$X$}; 
  \node[anchor=west] at (-2.2,5) {universe};

  \node (concept) at (-.4,4) {$c$};
  \fill[fill=au-cyan!40] (0,3.6) rectangle ++(5,.8);
  \node[] (concept-content) at (2.5,4) {$1~~1~~1\quad\ \ \quad\cdots\ \ \quad\quad 1~~1~~1$};
  \node[anchor=west] at (-2.2,4) {concept};

\begin{scope}[yshift=-0.05cm]
  \fill[fill=black!10,rounded corners] (-.75,-0.5) rectangle (5.1,3.5);

  \node (h0) at (-.4,3) {$h_0$};
  \fill[fill=au-cyan!40] (0,2.6) rectangle ++(3.7,0.8);
  \node[] (h0-1) at (1.8,3) {$1~~1~~1\quad\cdots\quad 1~~1~~1$};
  \fill[fill=au-orange!40] (3.7,2.61) rectangle ++(1.3,0.8);
  \node[] (h0-2) at (4.35,3) {-1\,-$1$\,-1}; 

  \node (h1) at (-0.4,2) {$h_1$};
  \node (h-dots) at (-.4,1) {\vdots};
  \node (hk) at (-.4,0) {$h_k$};

  \fill[fill=au-green!40] (0,2.4) rectangle (5,-0.4);

\node[anchor=west] at (0,2) {\small \,1{-\!1}{-\!1}{-\!1}1\,1};
\node[anchor=west] at (0,1.5) {\small {-\!1}{-\!1}\,1\,11{-\!1}};
\fill[au-green!40, path fading=fade example] (0,2.4) rectangle (1.5,1); 

\node[] (random) at (2.5,0.9) {fully random hypotheses};

\node (randomhyp) at (-1.25,1.5) {$\mathcal H$};
\draw[decorate,decoration={calligraphic brace,amplitude=7pt}, very thick] (-.8,-0.45) -- ++(0,3.8);
  \end{scope}

\end{tikzpicture}     \end{center}
    \caption{Illustration of our hypothesis set $\Hyp$}
    \label{fig:hypothesis}
\end{figure}

For the random construction of $\Hyp$, we sample at most $2^{d-1}$ hypotheses $h : \Xs \to \bool$ independently and uniformly at random. This clearly implies that the VC-dimension of $\Hyp$ is less than $d$. We now have to argue that we can use $\Hyp$ to design a $\gamma$-weak learner for the all-1 concept. 
Here, we distinguish two cases.
First, consider any distribution $\dist$ over $[u]$ where most of the probability mass is concentrated on some $r$ entries. Anti-concentration results imply that a random hypothesis has an advantage of $\Omega(\sqrt{\ln(1/\delta)/r})$ with probability at least $\delta$. We need the advantage to be at least $\gamma$ and we have $\exp(\Omega(d))$ hypotheses to choose from. 
Thus, if we plug in $\delta = \exp(-\Omega(d))$, we see that for $r = O(d/\gamma^2)$ we expect that the random $\Hyp$ contains a hypothesis with advantage $\gamma$ under $\dist$.
Thus, for distributions with small support, we can get a high advantage. A similar argument shows that we can at the same time get a \emph{negative} advantage of $-\gamma$ on $\bar{S}'$ as required earlier. However, AdaBoost might feed $\W$ a distribution $\dist_t$ that is not concentrated on some $O(d/\gamma^2)$ entries.
In this second case, we would intuitively like to add the all-ones hypothesis $h_0^\star$ to $\Hyp$ to achieve an advantage on such $\dist_t$.
Then $\W_{\Hyp}$ can always return $h_0^\star$ when being fed a distribution that is far from concentrated on a few entries. This is problematic for our lower bound since now AdaBoost could put a large weight on $h_0^\star$ which would cancel out any mistakes/negative advantage we accumulated in $\bar{S}'$.

To remedy this, we introduce the hypothesis $h_0$ which resembles $h_0^\star$ on most elements (returning $1$ there) but returns $-1$ on $c d/\gamma^2$ elements of $\Xs$ for some constant $c>0$.
Then, similar to $h_0^\star$, the hypothesis $h_0$ has a $\gamma$ advantage under all $\dist$ that are ``spread out'', i.e. do not have most of its mass on $O(d/\gamma^2)$ entries.
Thus, we can let $\W_{\Hyp}$ return $h_0$ for such $\dist$. 
If on the other hand $\dist$ is concentrated on few entries, we can find one of the random $h$ that has advantage at least $\gamma$ under $\dist$ and at most $-\gamma$ for a uniform element in $\bar{S}'$. 
But $\bar{S}'$ might be (mostly) among the coordinates where $h_0$ returns $1$. Thus, if AdaBoost puts too large a weight on $h_0$, then the negative advantage we accumulated on $\bar{S}'$ is still canceled out by $h_0$. This is where we use that $h_0$ has many $-1$'s. Concretely, we show that if $h_0$ receives a weight of more than some $O(\gamma)$, then there is no way to cancel out the $-1$'s that $h_0$ produces. In summary, if AdaBoost assigns a large weight to $h_0$ in its output classifier $f$, then $f$ makes negative predictions where $h_0$ is negative. If AdaBoost assigns a small weight to $h_0$, then $f$ makes negative predictions in $\bar{S}'$. In both cases, we have $\Omega(d/\gamma^2)$ negative predictions.
We illustrate this in \cref{pic:cases}.

\begin{figure}[h!]
    \begin{center}
    \includegraphics*[]{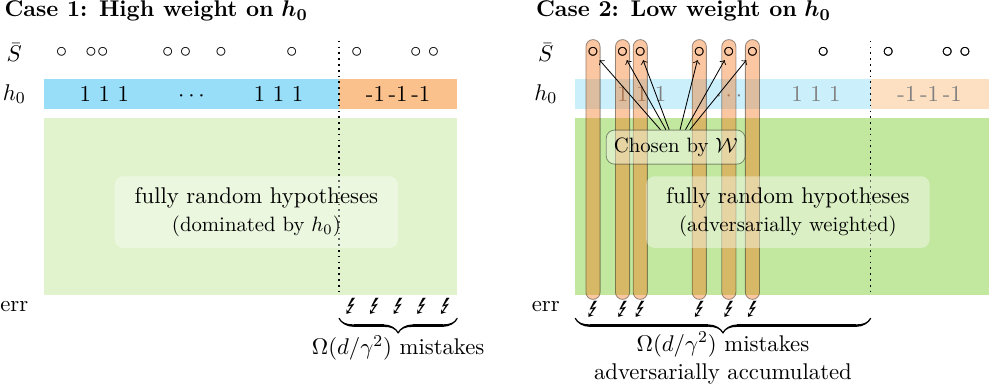}
    \end{center}
\caption{Illustration of where errors will occur}
    \label{pic:cases}
\end{figure}

Finally, let us summarize precisely what properties of AdaBoost we exploited above. As mentioned earlier, the key point is that the adversarial weak learner can determine the elements $\bar S$ of $\Xs$ that are not part of the training set $S$. It can thus return hypotheses that have a negative advantage of $-\gamma$ on some $\Omega(d/\gamma^2)$ elements of $\bar{S}$. This negative advantage is enough that it is not canceled out by any weight that AdaBoost assigns to a nearly all-1 hypothesis $h_0$. 
Note though that it is crucial that the negative advantage achieved by $\W$ is $-\gamma$ and not just negative as AdaBoost may use $h_0$ ``a little bit'', i.e. with a weight of up to some small constant times $\gamma$. 
If AdaBoost would put more weight on $h_0$, this would induce negative predictions where $h_0$ is negative.

Let us also remark that it is vital for our argument that every distribution $\dist_t$ fed to $\W$ by AdaBoost is non-zero on \emph{all} of the training data. Assume for instance that $\dist_t$ was only non-zero on a random constant fraction of $S$. Then the weak learner could only identify some random superset of $\bar{S}$ having linear size in $u$. But the weak learner needs to force a negative advantage of $-\gamma$ on some $\Omega(d/\gamma^2)$ points to cancel out the positive contributions by $h_0$.
Concentration results show that this can only be done on $O(d/\gamma^2)$ points and thus the adversarial weak learner would have to pick $O(d/\gamma^2)$ points among the random $\Omega(u)$ with zero mass under $\dist_t$. If these are in the training data $S$, which is the most likely case as $\bar{S}$ has cardinality only $\Theta(d/\gamma^2)$, then these $O(d/\gamma^2)$ points will have non-zero mass in most other $\dist_{t'}$, allowing a boosting algorithm to correct the negative predictions.

The above proof outline can be seen to work for any boosting algorithm producing voting classifiers and that always invokes the weak learner with a probability distribution that is strictly positive on all of the training data. For this reason, our lower bound argument also applies to many other classic boosting algorithms as mentioned in \cref{sec:intro}. In addition to showing that these algorithms are sub-optimal, we believe our lower bound may help inspire new boosting algorithms. Concretely, as just sketched above, if the weak learner was invoked with probability distributions that have mass on only a constant fraction of the training data, our argument breaks down. In fact, the optimal weak-to-strong learner by \citet{optimalWeakToStrong} precisely samples subsets of the training data and runs AdaBoost$^*_{\nu}$ on such subsets. 
Perhaps a similar sub-sampling could be used without the two-level majority. We leave this as an exciting direction for future research.

\section{AdaBoost is not Optimal}\label{maintheorem}\label{sec:maintheorem}
In this section, we prove our main result that AdaBoost is not an optimal weak-to-strong learner.

In the following, we let $\Xs =[u]= \{ 1,\dots,u \}$ be the universe where $u$ is the universe size. 
Further we let $\Delta_{\Xs}$ be the set of probability distributions over $\Xs$. 
In our construction, we use the all ones hypothesis, i.e. $h_0^\star(x) = 1$ for all $x\in \Xs$, as the underlying concept that is to be learned.
Since we do not consider any other concept, the error of a hypothesis $f$ under a distribution $\dist \in \Delta_{\Xs}$ is given as 
\[
    \mathcal{L}_{\dist}(f)=P_{x\sim \dist}\left[f(x)\not=1\right].
\] 
This is equivalent to 
\(\mathcal{L}_{\dist}(f)=\sum_{i=1}^u \dist(i)(1-f(i))/2\) such that we can write the error requirement of a $\gamma$-weak learner as
\[
    \sum_{i=1}^u \dist(i) f(i)\geq 2\gamma
\]
which we will use in the analysis.

In our construction, we will need the hypothesis $h_0$, which is ``close'' to the all ones hypothesis $h_0^\star$.
Let $h_0$ be the hypothesis from $\Xs$ into $\{-1,1\}$ such that ${h_0}(i)=1$ for $i=1,\ldots,u-r_1$ and ${h_0}(i)=-1$ for $i=u-r_1+1,\ldots,u$, for $r_1$ to be defined later (think of $r_1$ as small compared to $u$). 
Let $\mathcal{A}$ be any learning algorithm which takes as input a sample $S$ and a weak learner $\W$, and satisfies the following:
\begin{properties}\label{propalg}
$ \ $\\ [-1em]
\begin{enumerate}[nolistsep]
    \item $\mathcal{A}$ outputs a weighted majority classifier, i.e. a classifier of the form $\sign(\sum_{i} w_i h_i)$ where $w_i$ are non-negative weights with $\sum_{i}w_i=1$ and $h_i$ are hypotheses obtained from the weak learner $\W$. The weights $w_i$ only depend on the performance of the $h_i$'s on $S$ (i.e. $w_i$ may depend on $h_j(S)$ for $j\neq i$ but not on any $h_j(x)$ for $x \notin S$). 
\item In every query to the weak learner $\W$, the algorithm $\mathcal A$ provides a distribution $\Dist \in \Delta_{\Xs}$ with $\operatorname{supp}(\Dist) = S$  ($\Dist_i > 0$ for $i \in S$ and $0$ otherwise).
    \item The learning algorithm $\mathcal{A}$ provides the true labels to the items in the sample in its query to $\W$. \end{enumerate}   
\end{properties}

The conditions above are necessary and sufficient for our construction of the adversarial weak learner.
1) ensures that the learning algorithm actually uses the weak learner to compute the majority classifier with weights based only on the samples in $S$ (and not $\bar S$).
2) gives away the sample to the adversarial weak learner such that it can accumulate errors outside the sample i.e. on points in  $\bar S = \Xs\backslash S$.
And 3) ensures that the weak learner is always asked to learn the all ones hypothesis, so we only need to guarantee an advantage of $\gamma$ on that.
Under those conditions, $\dist$ already encodes $S$ such that we view the weak learner as a function of a distribution $\dist \in \Delta_{\Xs}$, instead of a function of $\dist$ and the sample $S$.
Furthermore, we write $\aW$ to make the hypothesis set $\Hyp$ that is used by a weak learner explicit.

The following lower bound is a more general version of \cref{thm:mainTheorem}.
Since AdaBoost satisfies the above properties, the lower bound applies to AdaBoost as well.


\begin{theorem}\label{thm:mainTheoremFormal}
   There exist a universal constant $c \leq 1$ such that for any $\gamma\leq c$, $d\geq  \ln(1/\gamma)$, $ d\gamma^{-2}/16\leq m\leq \exp(\exp(d))$ and learning algorithm $\mathcal{A}$ satisfying \cref{propalg},  there exist a universe $\Xs$, a distribution $\mathcal{D}\in \Delta_{\Xs}$, a hypothesis set $\Hyp$ of VC-dimension $O(d)$, and a weak learner $\aW$ on $\Hyp$ for the all one hypothesis i.e.
  \begin{align*}
    \forall \Dist \in\Delta_{\Xs}: \ \ \sum_{i\in \left[u\right]} \dist(i)\, \aW(\dist)(i) \geq 2\gamma,
  \end{align*}
  such that with constant probability over $S \sim \dist^m$:
  \begin{align*}
    \mathcal{L}_{\dist}(\mathcal{A}(S,\aW))=\Omega\left(\,\frac{d\ln \left(m\gamma^2/d)\right)}{m\gamma^2}\right)
  \end{align*}
\end{theorem}

Formally, \cref{thm:mainTheorem} follows from \cref{thm:mainTheoremFormal} by invoking it with $d' = O(d)$ (implying $m=\exp(\exp(O(d))$) and solving the loss $\mathcal{L}_{\dist}(\text{AdaBoost}(S,\aW)) = \eps=C\,\frac{d\ln \left(m\gamma^2/d\right)}{m\gamma^2}$ for $m$.

To prove \cref{thm:mainTheoremFormal} we use the following three lemmas whose proofs are deferred to \cref{sectionprooflemmas}.
The first lemma is a concentration inequality for linear combinations of independent, negatively biased $\{-1,1\}$-variables.
Notationwise, we denote a fixed hypothesis set by $\Hyp$ and a random one by $\sh$.
Similarly, a concrete hypothesis (which can be encoded by a vector) is denoted by $h$ and a random hypothesis by $\shh$.

\begin{restatable}{lemma}{biaslemma}\label{biaslemma}
    Let $w\in \mathbb{R}^d$ such that $\|w\|_1=1$ and let $\cre\geq 1$. 
    Let further $\shh$ be a random vector in $\{-1,1\}^d$ with i.i.d. entries such that $\pr\left[\shh(i)=1\right]=1/2-\cre\beta$ and  $\pr\left[\shh(i)=-1\right]=1/2+\cre\beta$ where $\beta< 1/(2\cre)$. 
    We then have for $\crt < \cre$ that 
    \begin{align*}
        \pr\left[\sum_{i=1}^dw_i\shh(i)\leq -\crt \beta \right]\geq \min\left(\frac{1}{4},\frac{1}{2}-\frac{4\cre\crt}{(2\cre-\crt)^2}\right).
    \end{align*}
\end{restatable}

The lemma will be used to get the $-\gamma$ advantages outside the sample $S$ as described in the proof overview.
The second lemma is of a coupon collector style. 

\begin{restatable}{lemma}{couponscollectors}\label{couponscollectors}
    Let $\cce m/\ln \left(m/r\right)$ be the number of coupons where $m\geq 4r$, $r\geq 1$, and $\cce\geq 8$. 
    Let $X$ denote the number of samples with replacement from the coupons before seeing $\cce m/\ln \left(m/r\right)-2r$ distinct coupons, then 
\(
     \pr\left[X\leq m\right]\leq \frac{1}{2}
     \)
\end{restatable}

In the proof we virtually split the universe into a main part and the last $r_1$ points and are interested in the probability of sampling a training set $S \in \spo :=\{S : |\bar S \cap [u-r_1]|\geq r\}$ (for some $r$ and $r \leq r_1$) capturing the case that there are  ``enough'' unsampled points in the main part of the universe.
We will use \cref{couponscollectors} and carefully chosen constants to show that this probability is at least constant.

The third lemma describes properties of two functions which we combine to get the random adversarial weak learner $\aWs$.

\begin{restatable}{lemma}{badweaklearners}\label{lemmabadweaklearners} 
  Let $\cmo,\cme \leq 1$, and $\cmt\geq 1$ denote universal constants. For a universe $\Xs$ of size $u$, integers $r,r_1$ with $r_1=\alpha^2 r$ for $\alpha\geq1$, and $\gamma\leq \cmo/(2\cg)$ there exist two independent random hypothesis sets $\he$ and $\htt$ such that
  \begin{itemize}

    \item For $\sh \coloneqq \he \cup \htt$ and $k=\ln \left(u\right)\gamma^{-2}$, \begin{align}\label{item5weaklearn}
        |\sh| \leq  4\cme^{-2}k\ln \left(k/\delta\right)\exp(8\cmt\gamma^2r_1)+1
      \end{align}
    \item 
        There exists a mapping $\gh:\distc\rightarrow \he$ such that for $r_1\geq 40\lg(|\he|)$ and $S\in \spo :=\{S : |\bar S \cap [u-r_1]|\geq r\}$,
        the mapping $\gh$ and the hypothesis set $\he$ satisfy the following four properties with probability at least $1-\delta-2^{-0.01r_1}$ (over the outcome of $\he$):
    \begin{enumerate}
      \item For any distribution $\Dist\in\mathcal{D}_S:=\{\dist:\dist(i)>0 \text{ for } i\in S \text{ else } \dist(i)=0, \|\dist\|_1=1\}$ supported on $S$, 
        $\sum_{i\in S} \Dist(i) \gh(\Dist)(i)\geq \gamma/4$.
      \label{item1weaklearn}
      \item 
      Let $\frs$ denote the first $r$ points from $\bar S\cap[u-r_1]$ and recall that $\Sss = S$.
      If for $\Dist \in\mathcal{D}_S$,  $\gh(\Dist) \neq h_0$, then the hypothesis $ \gh(\Dist)$ has $(1/2+\alpha\gamma/2)r$ minus signs in $\frs$.
Further, the outcome of $\gh(\Dist)$ on $\frs$ is uniformly distributed among all vectors in $\{-1,1\}^{r}$ which have at least $(1/2+\alpha\gamma/2)r$ minus signs.\label{lemmabadweklarnsitem2}
      \item The randomness over $\frs$ in \Cref{lemmabadweklarnsitem2} is independent for all hypotheses in $\{\gh(\dist) \text{ for } \dist\in\distc\}$. Further, the outcome of $\gh$ on $\frs$  is independent of $\gh$ on $\overline{\frs}$.  \label{item3weaklearn} 
      \item For any weight vector $w\in \Delta_{\he \backslash h_0}:=\{w\in\mathbb{R}^{|\he|}: 0\leq  w_i, w_0=0, \sum_{i\in|\he|}w_i=1 \}$ weighing the hypotheses in $\he$, we have for at least $r_1/10$ of the $i$'s in $\{u-r_1+1,\ldots,u\}$, that $\sum_{j\in |\he|} w_j h_j(i)\leq 14\sqrt{\log \left(|\he|\right)/r_1}$. \label{item4weaklearn} 
    \end{enumerate}

    \item There exists a mapping $\thh:\mathcal{D}\rightarrow \htt $ such that with probability at least $1-\delta$ over $\htt$, it holds for all $\dist\in \distc$ that \(
      \sum_{i\in [u]} \dist(i) \thh(\dist)(i)\geq \gamma/4.\label{item6weaklearn}
    \)
\end{itemize}
\end{restatable}

Let us carefully go over the statements in \cref{lemmabadweaklearners}.
The first bullet bounds the size of the hypothesis set $\sh$, ensuring that its VC-dimension is at most $O(d)$. 
The second and third bullet consider the functions $\gh$ and $\thh$ from which we construct the weak learner $\aWs$.
These functions, as well as $\aWs$, take as input a distribution and output a hypothesis from $\sh$.
The key idea is that whenever $\gh$ outputs a hypothesis with sufficient advantage on $\dist$, $\aWs$ will use that hypothesis (and therefore $\gh$ to compute it), and otherwise $\aWs$ will use the hypothesis computed by $\thh$. 
We thus think of $\thh$ as a safety mechanism that ensures that we can always get the required advantage $\Omega(\gamma)$ which is guaranteed by the last bullet of \cref{lemmabadweaklearners}.
We will call the lemma with $8\gamma$ instead of $\gamma$ to achieve an advantage of $2\gamma$.

With this in mind, consider the second bullet of \cref{lemmabadweaklearners} and consider some $S \in \spo$. Let us denote by $\eg$ the event that the four properties in the bullet hold for $S$ and $\he$. Now assume a weak-to-strong learning algorithm $\Alg$ that satisfies 1), 2), and 3) from \cref{propalg} and that receives an $S\in \spo$, i.e. at least $r$ of the unsampled points receive a positive label under the hypothesis $h_0$. Assume further that $\eg$ occurs. 
Then our weak learner $\aWs$ has the following interesting properties.

First, in this case, the weak learner $\aWs$ always returns a hypothesis produced by $\gh$. 
This holds as \cref{item1weaklearn} of the second bullet guarantees a sufficient advantage regardless of what distribution $\mathcal{A}$ queries the weak learner with.

From the second bullet's \cref{lemmabadweklarnsitem2} and \Cref{item3weaklearn}, we get that $\mathcal{A}$ can not put too much mass on the hypotheses provided by $\gh$ (those different from $h_0$), without making at least $\Omega(r)$ mistakes on the $r$ unsampled points $\frs$. 
These mistakes would imply an error of at least  $\Omega\left((d\ln (m\gamma^2/d))/(m\gamma^2)\right)$.

Finally, \Cref{item4weaklearn} gives us that $\Alg$ can neither put too much mass on $h_0$, without making  $\Omega(r)$ mistakes on the last $r_1$ points of $\Xs$. Combining this with the previous point gives the desired lower bound.
We now give the proof of \cref{thm:mainTheoremFormal}.

\begin{proof}[Proof of \cref{thm:mainTheoremFormal}]

Let $\gamma$, $d$, and $m$ be as in \cref{thm:mainTheoremFormal}. Let the concept that $\mathcal{A}$ is trying to learn be the all ones hypothesis $h_0^\star$. 
We now show the existence of a universe $\Xs$, a hypothesis set $\Hyp$ of VC-dimension at most $d$, and a $\gamma$-weak learner $\aW: \Delta_{\Xs} \rightarrow \Hyp$ for $h_0^\star$ (mapping distributions over $\Xs$ to hypotheses from $\Hyp$), such that when $\mathcal{A}$ uses hypotheses from $\aW$ and receives samples from the uniform distribution $\mathcal{U}$ on $\Xs$, then with constant probability over the sample $\samp \sim \mathcal{U}^m$, it has an error of $\mathcal{L}_{\mathcal{U}}(\mathcal{A}(S,\aW))=\Omega\left((d\ln \left(m\gamma^2/d)\right))/(m\gamma^2)\right)$.

To show the existence of such a hypothesis set and weak learner, we show for a random hypothesis set $\sh$ (with VC-dimension $O(d)$) that we have
\begin{align}\label{niceprop}
  &\E_{\sh}\bigg[\pr_{\samp}\bigg[\mathcal{L}_{\mathcal{U}}(\mathcal{A}(\samp,\aWs))\geq C\,\frac{d\ln \left(m\gamma^2/d)\right)}{m\gamma^2} ,\  \forall \dist\in \Delta_{\Xs}: \sum_{i\in[u]} \dist(i)\,\aWs(\dist)(i) \geq 2 \gamma\bigg]\bigg]\nonumber\\
  =~&\E_{\samp}\bigg[\pr_{\sh}\bigg[\mathcal{L}_{\mathcal{U}}(\mathcal{A}(\samp,\aWs))\geq C\,\frac{d\ln \left(m\gamma^2/d)\right)}{m\gamma^2} ,\ \forall \dist\in \Delta_{\Xs}: \sum_{i\in[u]} \dist(i)\,\aWs(\dist)(i) \geq 2 \gamma\bigg]\bigg]\geq \frac{1}{64}
\end{align}
for some universal constant $C$. 
Here, the first part states that $\mathcal A$ has a large error while the second part ensures that $\aWs$ is indeed a weak learner.
As the event of $\aWs$ being a weak learner is independent of $\samp$, the expectation implies that there exists a concrete hypothesis set $\mathcal H$ such that $\aW$ is a weak learner and with constant probability over the sample $\samp$, the algorithm $\mathcal{A}$ has error probability $\Omega\left(d\ln \left(m\gamma^2/d)\right))/(m\gamma^2)\right)$ when using $\aW$ as its weak learner. 
The equality uses that a probability can be written as the expectation of an indicator variable.

\paragraph{Establishing \Cref{niceprop}.}
Our adversarial weak learner accumulates errors on $r$ elements in $\bar S$, such that the overall error is connected to the fraction $r/u$.
Next, we show that we can invoke \Cref{lemmabadweaklearners} with parameters such that $r/u\geq C(d\ln \left(m\gamma^2/d)\right))/(m\gamma^2)$ for some universal constant $C$, and where $\thh$ is a weak learner with probability at least $1-\delta$ for $\delta=1/4$.
Using this, we can phrase \Cref{niceprop} as
\begin{align}\label{niceprop2}
    &\E_{\samp}\bigg[\pr_{\sh}\bigg[\mathcal{L}_{\mathcal{U}}(\mathcal{A}(\samp,\aWs))\geq \frac{r}{10u} ,\ \forall \dist \!\in\!\Delta_{\Xs}\!:\!\sum_{i\in[u]}\! \dist(i)\,\aWs(\dist)(i) \geq 2 \gamma\bigg]\bigg] > \frac{1}{64}.
\end{align}

We now show that such a choice of parameters is indeed possible.

\paragraph{Preliminary Setting of Parameters.}
Let $m\geq 8$ be the sample size and $\gamma' = 8 \gamma$ where $\gamma$ is the (sufficiently small) advantage needed for the weak learner.
This choice implies that the weak learner constructed in \cref{lemmabadweaklearners} has a $2\gamma$ advantage.

Now, let $u=\cmu m/\ln \left(m/r\right)$ be the universe size where $r:=d\gamma'^{-2}$ and where the value of $\alpha$ will be chosen larger than $1$. 
From the assumption $m\geq d\gamma^{-2}/16$ in the theorem we get that $m/r\geq 4$ and thus $\ln(m/r)$ is non-negative.
We now choose $r_1=\alpha^2 r=\alpha^2 d \gamma'^{-2}$.
In the definition of $h_0$ the last $r_1$ positions return $-1$, thereby splitting the universe in a ``first'' and ``second'' part.
Note that $u\geq \cmu m/\ln (m/r)\geq \cmu r\geq 8r_1$ (using that $x/\ln x >1$ for $x>1$ in the second inequality),
thus we may assume that the set of samples $\spo:=\{S:|\bar S\cap [u-r_1]|\geq r\}$ is not $\emptyset$. 

We wish to invoke \Cref{lemmabadweaklearners} with $u$, $\gamma=\gamma'$, $r$, $\alpha$, and $\delta=1/4$ as above.
First, \Cref{item5weaklearn} guarantees that the size of $\sh$ in \Cref{lemmabadweaklearners} is upper bounded by $4\cmo^{-2}k\ln \left(k/\delta\right)\exp(8\cmt\gamma'^2r_1)+1\leq 5\cmo^{-2}k\ln \left(k/\delta\right)\exp(8\cmt\gamma'^2r_1)$.
\Cref{lemmabadweaklearners} only holds when $\gamma'\leq \cmo/(2\alpha)$. 
We guarantee this with the constraint in \cref{thm:mainTheoremFormal} saying that $\gamma\leq c$, where $c$ is less than $\cmo/(16\alpha)$.

We now decide on the choice of $\alpha$.
Later in the proof, we will need that $\ln(|\sh|)/r_1\leq \cmf \gamma^2$ where $\cmf$ is a universal constant that will determine the concrete value of $\alpha$.
To upper bound  $\ln(|\sh|)/r_1$, we first notice that since $m\leq \exp(\exp(d))$ we get that $\ln(\ln(u))\leq \ln(\ln(8\alpha^2m))\leq \ln(\ln(8\alpha^2))+d$.
Further, since $d\geq \ln(1/\gamma)$ (one of the conditions in \cref{thm:mainTheoremFormal}) we get that $\ln(k)=\ln(\ln \left(u\right)\gamma'^{-2})\leq\ln(\ln(8\alpha^2))+3d$.
By these two inequalities as well as $\delta= 1/4$ and $r_1=\alpha^2d\gamma'^{-2}$ we get that
\begin{align}
   \ln (|\sh|)
    \leq 8\cmt\gamma'^2r_1+\ln (5\cmo^{-2}) +\ln(k) +\ln(\ln \left(k/\delta\right)) 
    \leq \ln (5\cmo^{-2})+5(\ln(\ln(8\alpha^2))+(8\cmt \alpha^2+3)d.\label{lnhypothesissize}
\end{align}
implying that for any $\alpha,d\geq 1$ if we choose $\cmf=\big(\ln(5\cmo^{-2})+5\ln(\ln(8))+(8\cmt +3 )\big)8^2$
\begin{align}\label{hypothesisratio}
\frac{\ln (|\sh|)}{r_1}
    \leq \left( \ln (5\cmo^{-2}) + \frac{5 \ln(\ln(8\alpha^2))}{\alpha^2 d}+8\cmt \!+\!\frac{3}{\alpha^2}\! \right)8^2\gamma^2
    \leq\cmf \gamma^2
\end{align}
since the middle expression in \Cref{hypothesisratio} is decreasing in $\alpha,d\geq 1$. 
This allows us to fix $\alpha=5\cdot 28\sqrt{\cmf}$.

Further notice that \cref{hypothesisratio}, the before mentioned constraint $\gamma\leq \cmo/(16\alpha)$, $\cmo\leq 1$ implied by \cref{lemmabadweaklearners}, and the now fixed $\alpha=5\cdot28\sqrt{\cmf}$, $\cmf\geq 1$ implies that $r_1\geq \ln(|\sh|)/(\cmf\gamma^2)\geq 40\lg(|\he|) $. This a condition for the second bullet of \cref{lemmabadweaklearners} to hold. 
We thus have that we can invoke \cref{lemmabadweaklearners} as claimed.

\paragraph{Bounded VC-Dimension.}
Using the parameters we have chosen above, we can now bound the VC-dimension of $\sh$.
Here we use that the VC-dimension of $\sh$ is trivially bounded by  $\ln|\sh|/\ln(2)$.
Together with the size bound on $\sh$ from \cref{lnhypothesissize} we get that the VC-dimension of $\sh$ is $O(d)$ as claimed.

We now construct our weak learner $\W$ in the following way using $\gh$ and $\thh$ from \Cref{lemmabadweaklearners}.
\begin{align*}
  \W_{\sh}(\dist) &= \indicator{}_{\sum_{i=1}^u \dist(i)\gh(\dist)(i)\geq 2\gamma }\,\gh\\
  &\,+\indicator{}_{\sum_{i=1}^u \dist(i)\gh(\dist)(i)<2\gamma }\,\thh(\dist) \quad \forall \dist\in\distc,
\end{align*}
Said in words, $\aWs$ is $\gh$ when $\gh$ achieves an advantage of $2\gamma$ and it defaults back to $\thh$ otherwise.\mrrand{I don't see the benefit from the formal definition of $\aWs$ and would just drop it.}

First, we notice that if $\thh$ is a weak learner, then $\aWs$ is also a weak learner. 
Thus we can replace the weak learning requirement on $\aWs$ by a similar requirement on $\thh$, implying
\begin{align*}
  \E_{\samp}\bigg[\pr_{\sh}\bigg[&\mathcal{L}_{\mathcal{U}}(\mathcal{A}(\samp,\aWs))\geq \frac{r}{10u} ,\ 
  \forall \dist\in \Delta_{\Xs}: \sum_{i\in[u]} \dist(i)\,\aWs(\dist)(i) \geq 2 \gamma\bigg]\bigg]\\
  \geq ~\E_{\samp}\bigg[\pr_{\sh}\bigg[&\mathcal{L}_{\mathcal{U}}(\mathcal{A}(\samp,\aWs)) \geq \frac{r}{10u},
  \ \forall \dist\in \Delta_{\Xs}: \sum_{i\in[u]} \dist(i)\,\thh(\dist)(i) \geq 2 \gamma\bigg]\bigg].
\end{align*}
Further notice that if we have a sample $S$, then $\Alg$ would by Item 2) in \cref{propalg} only give inputs $\dist$ in $\mathcal{D}_S:=\{\dist:\dist(i)>0 \text{ for } i\in S \text{ else } \dist(i)=0, \|\dist\|_1=1,\}$ to the weak learner $\aWs$.
Thus, we have for a fixed sample $S$ and the definition of $\aWs$  that 
\begin{align*}
  &\Big\{\Hyp=(\Hyp^1 \cup \Hyp^2) :
    \mathcal{L}_{\mathcal{U}}(\mathcal{A}(S,g_{\Hyp^1}))\geq \frac{r}{10u}, 
    \ \forall \dist \in \mathcal{D}_S \sum_{i\in[u]} \dist(i)\,g_{\Hyp^1}(\dist)(i) \geq 2 \gamma
  \Big\} \\
  \subseteq   &\left\{\Hyp :
    \mathcal{L}_{\mathcal{U}}(\mathcal{A}(S,\Ws_{\Hyp}))\geq \frac{r}{10u}
  \right\}
\end{align*}
where we use that $\aWs$ becomes $\gh$ when $\gh$ produces large margins.
Thus, we conclude that
\begin{align}\label{importantthree}
  \E_{\samp}\bigg[\pr_{\sh}\bigg[&\mathcal{L}_{\mathcal{U}}(\mathcal{A}(\samp,\aWs)) \geq \frac{r}{10u}, \ \forall \dist\in \Delta_{\Xs}: \sum_{i\in[u]} \dist(i)\,\thh(\dist)(i) \geq 2 \gamma\bigg]\bigg]\nonumber\\
  \geq
  \E_{\samp}\bigg[\pr_{\sh}\bigg[&\mathcal{L}_{\mathcal{U}}(\mathcal{A}(\samp,\gh))\geq \frac{r}{10u},
  \ \forall \dist \in \mathcal{D}_{\samp}: \sum_{i\in[u]} \dist(i)\,\gh(\dist)(i) \geq 2 \gamma, \nonumber\\
  &\qquad\qquad\qquad\qquad\qquad\ \forall \dist\in \Delta_{\Xs}: \sum_{i\in[u]} \dist(i)\,\thh(\dist)(i) \geq 2 \gamma\bigg]\bigg]\nonumber\\
  \geq 
  \E_{\samp}\bigg[\pr_{\sh}\bigg[&\mathcal{L}_{\mathcal{U}}(\mathcal{A}(\samp,\gh))\geq \frac{r}{10u},\ \forall \dist \in \mathcal{D}_{\samp} \sum_{i\in[u]} \dist(i)\,\gh(\dist)(i) \geq 2 \gamma\bigg]\bigg](1-\delta)
\end{align}
where the last inequality follows from the last point of \Cref{lemmabadweaklearners}, which says that $\thh$ is a weak learner with probability at least $1-\delta$ and $\thh$ is independent of $\gh$.

We will now show that
\begin{align}\label{importanttwo}
  \pr_{\samp}[\samp\in \spo]\geq 1/4,
\end{align}

and for any sample $S$ in the set $\spo:=\{S:\bar{S}\cap [u-r_1]|\geq r\}$ (from \Cref{lemmabadweaklearners}) we have that 
 \begin{align}\label{importantone}
    \pr_{\sh}\bigg[&\mathcal{L}_{\mathcal{U}}(\mathcal{A}(S,\gh))\geq \frac{r}{10u},  
    \ \forall \dist \!\in\! \mathcal{D}_S\!:\!\sum_{i\in[u]} \dist(i)\,\gh(\dist)(i) \geq 2 \gamma\bigg]\geq \frac{1}{12}.
 \end{align}

Now combining \Cref{importantthree},  \Cref{importanttwo}, \Cref{importantone}, and $\delta=1/4$ we get
\begin{align*}
    \E_{\samp}\bigg[\pr_{\sh}\bigg[&\mathcal{L}_{\mathcal{U}}(\mathcal{A}(\samp,\aWs))\geq \frac{r}{10u} ,
    \ \forall \dist\! \in\! \Delta_{\Xs}\!:\! \sum_{i\in[u]} \dist(i)\,\aWs(\dist)(i) \geq 2 \gamma\bigg]\bigg]
    \geq  \frac{1}{64}
\end{align*}
as desired.
Thus, if we can show \Cref{importanttwo} and \Cref{importantone}  we are done. 
Essentially, \cref{importanttwo} makes sure that we (often enough) have space in $\bar S$ to accumulate errors using $\gh$.
\cref{importantone} gives us that if there is space to accumulate errors, many of the random hypothesis sets $\sh^1$ allow us to actually do so.
\cref{importantthree} accounts for the behavior of the weak learner, i.e. its decision rule between the adversarial function $\gh$ and the `normal' weak learner $\thh$.

\paragraph{Establishing \cref{importanttwo}:}
Recall that we chose the universe size to be $u=8\alpha^2m/\ln(m/r)$ and the sample distribution to be uniform on $\Xs=[u]$ (corresponding to drawing with replacement from $\Xs$). Further we had $r=d\gamma'^{-2}$ which by the assumption $m\geq d\gamma^{-2}/16$, implied that $m/r\geq 4$. 
Using this, we get from \Cref{couponscollectors} with $\zeta=8\alpha^2\geq 8$  that with probability at least $1/2$ there are $2r$ points in $u$ that are not sampled into $\samp$.
Further, by the choice of $r_1=\alpha^2 r$ we noticed that $u=8\alpha^2m/\ln(m/r)\geq 8r_1$ thus the universe has at least $8$ times the size of $r_1$. 
Using this together with $r\leq r_1$ (since $\alpha\geq 1$) and the sampling distribution being uniform/with replacement, we conclude that at least half of the samples where $2r$ points were not sampled in $\Xs$ have $r$ entries outside of $\{u-r_1+1,\ldots,u\}$ implying $r\leq |\bar{S}\cap[u-r_1]|$, i.e. $S\in\spo$.
Thus, we conclude that $\pr_{\samp}\left[\samp\in \spo\right]\geq\pr_{\samp}\left[|\samp|\leq u-2r\right]/2\geq 1/4$ which shows \Cref{importanttwo}. 

\paragraph{Establishing \cref{importantone}:}
For \cref{importantone} let $S$ be in $\spo$ and notice that by \Cref{lemmabadweaklearners} we have with probability at least $1-\delta-2^{-0.01r_1}$ over $\sh$ that all the 4 items regarding $\gh$ in \cref{lemmabadweaklearners} hold.
Let $\eg$ denote the corresponding event that those 4 properties regarding $\gh$ in \cref{lemmabadweaklearners} hold.
In particular, \cref{item1weaklearn} says that $\gh$ is indeed a weak learner on $\mathcal{D}_S$. 
Using this event $\eg$ we get that
\begin{align}\label{importantfour}
    &\pr_{\sh}\bigg[\mathcal{L}_{\mathcal{U}}(\mathcal{A}(S,\gh))\geq \frac{r}{10u}, 
    ~\quad\forall \dist \!\in\! \mathcal{D}_S \!:\!\sum_{i\in[u]} \dist(i)\,\gh(\dist)(i) \geq 2 \gamma\bigg] \nonumber \\ 
    \geq ~&\pr_{\sh}\!\left[\mathcal{L}_{\mathcal{U}}(\mathcal{A}(S,\gh))\geq \frac{r}{10u}\,\middle|\,\eg\right]\!(1-\delta-2^{-0.01r_1}).
\end{align}
We now show that conditioned on $\eg$, with probability at least $1/6$ the algorithm $\Alg$ has an out-of-sample error of at least $r/(10u)$ when using $\gh$ as the weak learner, formally $\pr_{\sh}\left[\mathcal{L}_{\mathcal{U}}(\mathcal{A}(S,\gh))\geq \frac{r}{10u}|\eg\right]\geq 1/6$. 
We further show that $1-\delta-2^{-0.01r_1}\geq 1/2$  which combined with \Cref{importantfour} implies \Cref{importantone}. 

By the definition of the event $\eg$ we know we know that in the event $\eg$ the random hypothesis set $\sh$ satisfies the 4 items of the second bullet of \Cref{lemmabadweaklearners} (the ones about $\gh$). 
Thus, $\gh$ is a weak learner on $S$ by \Cref{item1weaklearn} and $\Alg$ terminates using only hypotheses given by $\gh$, which satisfy the conditions given in the 4 items.
Let  $w^{\mathcal{A}}=(w^{\mathcal{A}}_0,\ldots,w^{\mathcal{A}}_{|\he|})$ be the weights that $\Alg$ calculates, where $w^{\mathcal{A}}_0$ is the weight put on $h_0$.
Notice that the weights are random as they depend on the outputs of $\gh$ which themselves depend on the random hypothesis set $\sh$. 
From the first item of the second bullet in \cref{lemmabadweaklearners} we know that the weights $w^{\mathcal{A}}$ depend only on $\gh(\cdot)(i)$ for $i\in S$.
Thus, we get by \Cref{lemmabadweklarnsitem2} and \Cref{item3weaklearn} in \cref{lemmabadweaklearners} that the minus signs of $\gh$ in the first $r$ points of $\bar{S}\cap[u-r_1]$, which we denoted as $\frs$, are independent of the weights $w^{\mathcal{A}}$.
We will use this property below in the second case.
In the following let $\{\shh_i\}_{i=1,\ldots,|\he|}$ be the hypotheses in $\he$.
Note that whenever a hypothesis $\shh_i$ has a positive weight $w_i>0$, there must be a distribution $\dist \in \distc$ such that $\gh(\dist)=\shh_i$.
We now consider two cases for the weight $w_0^{\mathcal A}$ of the all-one hypothesis $h_0$. 
For this let $E_\text{small}$ be the event that $w^{\mathcal{A}}_0 < 14\sqrt{\cmf\gamma^2}/(1+14\sqrt{\cmf\gamma^2})$.

\paragraph{Case 1: $w^{\mathcal{A}}_0 \geq 14\sqrt{\cmf\gamma^2}/(1+14\sqrt{\cmf\gamma^2})$ ($\overline{E_\textnormal{small}}$).}

Consider the $r_1$ last points in the universe $\Xs=\left[u\right]$, i.e. the points where $h_0$ is $-1$. 
Thus, for $i\in \{u-r_1+1,\ldots,u\}$ we have that the prediction of $\Alg$ is $w^{\mathcal{A}}_0h_0(i)+\sum_{j=1}^{|\he|}w^{\mathcal{A}}_j\shh_j(i)=-w^{\mathcal{A}}_0+(1-w^{\mathcal{A}}_0)\sum_{j=1}^{|\he|}w^{\mathcal{A}}_j/(1-w^{\mathcal{A}}_0)\shh_j(i)$, where we have used that $h_0(i)=-1$ for $i\in \{u-r_1+1,\ldots,u\}$.
Now, conditioned on $\eg$ we know by \Cref{item4weaklearn} in \Cref{lemmabadweaklearners} that for any weighted combination of $(\shh_j)_{1,\ldots,|\he|}$ there are at least $r_1/10$, $i$'s in $\{u-r_1+1,\ldots,u\}$ where the linear combination is at most $14\sqrt{\log(|\he|)/r_1}$ i.e. for such $i$'s we have $\sum_{j=1}^{\he}w_j \shh_j(i)\leq 14\sqrt{\log(|\he|)/r_1}$.
By \Cref{hypothesisratio} and $|\he|<|\sh|$ we know that $14\sqrt{\log(|\he|)/r_1}$ is strictly less than $14\sqrt{\cmf\gamma^2}$. 
Thus, we get for such elements $i$ that $-w^{\mathcal{A}}_0+(1-w^{\mathcal{A}}_0)\sum_{j=1}^{|\sh|}w^{\mathcal{A}}_j/(1-w^{\mathcal{A}}_0)\shh_j(i)< -w^{\mathcal{A}}_0+(1-w^{\mathcal{A}}_0)14\sqrt{\cmf\gamma^2}$, which for $w^{\mathcal{A}}_0 \geq 14\sqrt{\cmf\gamma^2}/(1+14\sqrt{\cmf\gamma^2})$ is  less than zero. 
Thus, conditioned on $\eg$, if $\Alg$ puts more than $ 14\sqrt{\cmf\gamma^2}/(1+14\sqrt{\cmf\gamma^2})$ mass on $w^{\mathcal{A}}_0$, then $\Alg$ gets at least $r_1/10\geq r/10$ points misclassified, resulting in an out of sample error of at least $r/(10u)$.
Thus, we conclude that 
\begin{align}\label{case1}
    \pr_{\sh}\!\left[\mathcal{L}_{\mathcal{U}}(\mathcal{A}(S,\gh))\!\geq\! \frac{r}{10u},\overline{E_\text{small}}\,\middle|\,\eg\right] 
    =~ \pr_{\sh}\left[~\overline{E_\text{small}}~\middle|~\eg\,\right].
\end{align}

\paragraph{Case 2: $w^{\mathcal{A}}_0 < 14\sqrt{\cmf\gamma^2}/(1+14\sqrt{\cmf\gamma^2})$ ($E_\textnormal{small}$).}
Let $R$ be the set of all indices of hypotheses with nonzero weights in $\wa$ except the index of $h_0$. 
Notice that $R$ depends on the vector if weights $\wa$ which depends on the random hypothesis set $\sh$, making $R$ random too.
Further, by the comments before Case 1, $i\in R$ implies that $\exists \dist \in \distc$ such that $\gh(\dist)=\shh_i$. 
Thus, we have by \Cref{lemmabadweklarnsitem2} of \Cref{lemmabadweaklearners} that for every $j\in R$ the vector $(\shh_j(i))_{i\in \frs}$ corresponds to a random vector of length $r$ with at least  $(1/2+8\cg\gamma/2)r$ minus signs and the vector is uniformly distributed between all permutations of $\{-1,1\}^r$ with at least $(1/2+8\cg\gamma/2)r$ minus signs (where we used $\gamma'=8\gamma$).
Further, \Cref{item3weaklearn} of \Cref{lemmabadweaklearners} states that these vectors (one for each hypothesis $j\in R$) are independent of each other and of $(\shh_j(i))_{j\in R, i\in \overline{\frs}}$, which the weights $w_i$ are a function of.
Therefore, the vectors $(\shh_j(i))_{i\in \frs}$ for $j\in R$ are also independent of the weights.
If we now let $\tilde{w}_j^{\mathcal A}:=w^{\mathcal A}_j/(1-w^{\mathcal A}_0)$ for $j\in R$ and use that for every $i\in \frs$ we know $h_0(i)=1$, we get for $i\in \frs$ that 
\begin{align}\label{smartbias}
    &\pr_{\sh} \bigg[w^{\mathcal A}_0h_{0}(i)+\sum_{j=1}^{|\he|}w^{\mathcal A}_j\shh_j(i)< 0,E_\text{small}~\bigg|~ \eg \bigg]\nonumber\\
    =\,&\pr_{\sh} \bigg[\sum_{j\in R}\tilde{w}^{\mathcal A}_j\shh_j(i)< -w^{\mathcal{A}}_0/(1-w^{\mathcal A}_0),E_\text{small}~\bigg|~ \eg \bigg]\nonumber \\
    \intertext{Since $-x/(1-x)$ is decreasing for $0\leq x\leq 1$ and we have $w^{\mathcal A}_0 < 14\sqrt{\cmf\gamma^2}/(1+14\sqrt{\cmf\gamma^2})$, which implies  $-w_0^{\Alg}/(1-w_0^{\Alg})> -14\sqrt{\cmf\gamma^2}$ and we get that}
    \geq\, &\pr_{\sh} \bigg[\sum_{j\in R}\tilde{w}^{\Alg}_j\shh_j(i)\leq-14\sqrt{\cmf}\gamma,E_\text{small}~\bigg|~ \eg \bigg]\nonumber\\
    \intertext{Now using the law of total probability gives us}
    =&\int\pr_{\sh} \bigg[\sum_{j\in R}\tilde{w}^{\Alg}_j\shh_j(i)\leq -14\sqrt{\cmf}\gamma,E_\text{small}~\bigg|~ \eg,\tilde{w}^{\Alg}=z \bigg] \nonumber \\
        &\qquad d\pr_{\sh}\left[\tilde{w}^{\Alg}=z\mid \eg\right]\nonumber \\
    =&\int_{E_\text{small}}\pr_{\sh} \bigg[\sum_{j\in R}\tilde{w}^{\Alg}_j\shh_j(i)\leq -14\sqrt{\cmf}\gamma~\bigg|~ \eg,\tilde{w}^{\Alg}=z \bigg] \nonumber \\
        &\qquad d\pr_{\sh}\left[\tilde{w}^{\Alg}=z\mid \eg\right]
\end{align}
We will now work towards lower bounding  $\pr_{\sh} \left[\sum_{j\in R}\tilde{w}^{\Alg}_j\shh_j(i)\leq -14\sqrt{\cmf}\gamma~\middle|~ \eg,\tilde{w}^{\Alg}=z \right]$ by $1/4 $ for any $z\in E_\text{small}$.
As noted above, we have for $i\in \frs$ that $(\shh_j(i))_{j\in R}$  are $-1$ with probability at least $1/2+4\alpha\gamma$, independent of each other and independent of the weights $w^{\Alg}$.
Thus, using that we chose $\alpha=5\cdot28\sqrt{\cmf}$ and by invoking \Cref{biaslemma} with $\cre=4\alpha=4\cdot5\cdot28\sqrt{\cmf}$ and $\alpha'=14\sqrt{\cmf}$, we get that 
\begin{align*}
     \frac{4\cre \alpha'}{(2\cre - \alpha')^2} 
    =\frac{4(4\cdot5\cdot28\sqrt{\cmf})\cdot (14\sqrt{\cmf})}{\big(2\cdot 4\cdot5\cdot28\sqrt{\cmf}-  14\sqrt{\cmf}\big)^2} 
    =\frac{4^2 \cdot5\cdot 2}{(2\cdot4\cdot5\cdot 2-1)^2}
    \leq \frac{1}{4}.
\end{align*}
Thus, $\min\left(\frac{1}{4},\, \frac{4\cre \alpha'}{(2\cre - \alpha')^2}\right)$ in \cref{biaslemma} is realized by $\frac{1}{4}$ and we get
\begin{align}\label{failbiasmain}
    \pr_{\sh} \bigg[\sum_{j\in R}\tilde{w}^{\Alg}_j\shh_j(i)\leq -14\sqrt{\cmf}\gamma~\bigg|~ \eg,\tilde{w}^{\Alg}=z \bigg]\geq \frac{1}{4}.
\end{align}
Notice that the condition $\gamma \leq 1/(2\cre) = 1/(8\alpha)$ of \cref{biaslemma} is already satisfied since we already imposed the condition $\gamma \leq \cmo/(16\alpha)$ with $\cmo \leq 1$ in the main theorem in order to apply \cref{lemmabadweaklearners}.

We now consider the error of the points in $\frs$, or more specifically, the part of the total error that is induced by points from $\frs$.
We get the following upper bound by observing that there are $r$ points in $\frs$:
\begin{align*}
  E_{\frs}~=~&(1/u)\sum_{i\in \frs}\indicator{}_{\sign\left({\sum_{j=0}^{|\he|}w^{\mathcal{A}}_j\shh_j(i)}\right) \not=1}\\
  =~&(1/u)\sum_{i\in \frs}\indicator{}_{\sum_{j=0}^{|\he|}w^{\mathcal{A}}_j\shh_j(i) < 0}\\
  \leq~&  r/u.
\end{align*}
By \Cref{failbiasmain} we get that $\E_{\sh}\left[E_{\frs}\mid \eg,\tilde{w}^{\Alg}=z \right]\geq r/(4u)$. 
This allows us to use a reverse Chernoff bound from which we get that 
\begin{align}\label{errorfr}
  \pr_{\sh}\left[E_{\frs}\geq r/(10u)~\middle|~ \eg,\tilde{w}^{\Alg}=z \right]
  \geq \frac{r/(4u)-r/(10u)}{ r/u-r/(10u)}
  = \frac{1/4-1/10}{1-1/10}
  = 1/6.
\end{align}
Using that $\mathcal{L}_{\mathcal{U}}\geq E_{\frs}$, \Cref{errorfr}, and following calculations as in \Cref{smartbias} we conclude that
\begin{align}\label{case2}
    &\pr_{\sh}\left[\mathcal{L}_{\mathcal{U}}(\mathcal{A}(S,\gh))\geq \frac{r}{10u},E_\text{small}~\middle|~\eg\right] \nonumber \\
    \geq~     &\pr_{\sh}\left[E_{\frs}\geq \frac{r}{10u},E_\text{small}~\middle|~\eg\right]\nonumber\\
    =~ &\int_{E_\text{small}}\pr_{\sh} \left[E_{\frs}\geq r/(10u)~\middle|~ \eg,\tilde{w}^{\Alg}=z \right]\nonumber \\
     &\qquad\quad d\pr_{\sh}\left[\tilde{w}^{\Alg}=z\mid \eg\right]\nonumber \\
    \geq ~ &\pr_{\sh}\left[E_\text{small}~\middle|~ \eg\right]/6
\end{align}

\paragraph{Combining the two cases:}
Now using \Cref{case1} and \Cref{case2} we get that
\begin{align}
    \pr_{\sh}\left[\mathcal{L}_{\mathcal{U}}(\mathcal{A}(S,\gh))\geq \frac{r}{10u}~\middle|~\eg\right]\geq 1/6.
\end{align}
Combining this with \Cref{importantfour} we conclude that
\begin{align}
    \pr_{\sh}\bigg[\mathcal{L}_{\mathcal{U}}(\mathcal{A}(S,\gh))\geq \frac{r}{10u}, 
    \forall \dist \in \mathcal{D}_S\!:\! \sum_{i\in[u]} \dist(i)\,\gh(\dist)(i) 
    \geq 2 \gamma\bigg]
    ~\geq~ \frac{1}{6}(1-\delta-2^{-0.01r_1}),
\end{align}
that is, for any $S\in \spo$, the function $\gh$ is a weak learner on $\mathcal{D}_S$ and $\Alg$ using $\gh$ makes at least $r/(10u)$ errors with probability at least $(1-\delta-2^{-0.01r_1})/6 $ over the random hypothesis set $\sh$. 
Now, since $\gamma\leq1/ (16\alpha)$, $\cmf\geq 1$, and $\alpha=5\cdot 28\sqrt{\cmf}$ we get that $r_1=\alpha\ln(m)/(8\gamma)^2\geq 4\alpha^3\ln \left(m\right)\geq 4\dot(5\cdot28)^3\ln \left(m\right)$.
Using $m\geq 2$\mrrand{don't we need that $r_1\geq 1$? Why is $m=2$ enough for that?} we get that $2^{-0.01r_1}\leq 1/4 $ and since we chose $\delta= 1/4$ we get that 
\begin{align*}
    \pr_{\sh}\bigg[\mathcal{L}_{\mathcal{U}}(\mathcal{A}(S,\gh))\geq \frac{r}{10u}, 
    \forall \dist \in \mathcal{D}_S\!:\! \sum_{i\in[u]} \dist(i)\,\gh(\dist)(i) \geq 2 \gamma\bigg]
    \geq \frac{1}{12},
\end{align*}
which shows \Cref{importantone} and concludes the proof.

\end{proof}

\section{Proof of Lemmas}\label{sectionprooflemmas}

In this section, we restate the lemmas from \cref{sec:maintheorem} and give their proofs.
A main part of the proof in \Cref{maintheorem} makes use of the functions $\gh$ and $\thh$ which on the random hypothesis set $\sh$ have ``nice'' properties (\cref{lemmabadweaklearners}). 
As $\gh$ and $\thh$ played the main role in \Cref{maintheorem} we start off by proving \cref{lemmabadweaklearners}.
To prove the lemma, we need the following algorithm which we use to show the existence of a the hypotheses $\gh$ and $\thh$ will output.

\begin{algorithm2e}[H]
    \DontPrintSemicolon
    \KwIn{($\mathcal{H}_1,\ldots,\mathcal{H}_k$), $S\subset \Xs$}
    \KwOut{ $f$ adversarial weak learner on $S$}

    $\eta \gets \ln \left(\left(1+2\gamma\right)/\left(1-2\gamma\right)\right)/2$

    $f_0(i)\gets0$ for all $i\in u$

    $D_1(i)\gets \frac{1}{S}$ for all $i\in S$
  
    \For{$j \in \{1,\ldots,k\}$}{
          
        \If{$\sum_{i=1,i\in S}^{u-r_1}D_j(i)> 1/2+\gamma$}{
            set $h_j=h_0$ (notice that if this is the case then $\sum_{i\in S} D_j(i) h_j(i)\geq 2\gamma$) \label{choosehzero}
        }
        \ElseIf{there is a hypothesis $h_j\in H_j$ such that $\sum_{i\in S} D_j(i)h_j(i)\geq 2\gamma$ and $h_j$ has $(1/2+\alpha\gamma/2)r$ minus signs on the first $r$ elements in $\bar{S}\cap[u-r_1]$}{ choose this hypothesis\label{chooseminuses}}
        \Else{
            \Return Fail \label{choosefail}
        }
        
        $f_j\gets f_ {j-1}+h_j$
        
        $Z_j\gets \sum_{i\in S} D_j(i)\exp\left(-\eta h_j(i)\right)$\label{algomultinorma}
        } \For{ $i\in S$}{
        $D_{j+1}(i)\gets D_j(i)\exp\left(-\eta h_j(i)\right)/Z_j$
    }
    \Return {$f=f_k/k$}
    \caption{Majority Voter}\label{majorityvoter}\label{alg:majorityvoter}
\end{algorithm2e} 
In the following proof of \cref{lemmabadweaklearners} we will run the above algorithm on a sequence of random hypothesis sets whose union will be $\sh$. 
Running the above algorithm will then create a voting classifier with a $\gamma$ advantage which implies that one of the hypotheses also has this advantage. 
Thus, $\sh$ contains a hypothesis with a $\gamma$ advantage that $\gh$ or $\thh$ can output.
In the case of $\gh$ we will also make these hypotheses adversarial by using the minus signs in \cref{chooseminuses}.
For the above argument to go through we need that the random hypothesis set $\sh$ contains at least one hypothesis that has a $\gamma$ advantage given a distribution $\dist$ over the universe $\Xs$ (for all distributions $\dist$ that the algorithm computes).
This is captured in the following lemma, which we will prove later in this section.

\begin{restatable}{lemma}{lemmabadhypothesis}
    \label{lemmabadhypothesis}
    Let $\cmo,\cme \leq 1$, and $\cmt\geq 1$ denote some universal constants. 
    Let $\Xs$ be a universe of size $u$ and $\dist \in \distc$ a distribution over $\Xs$. 
    Further let  $r$ and $r_1$ be non-negative numbers such that $r_1=\alpha^2 r$ for $\alpha\geq 1$ and $r_1\leq u$. 
    Let $0<\delta\leq 1$,  $\gamma\leq \cmo/(2\cg)$, and $k=\ln \left(u\right)\gamma^{-2}$.  Let $\hi$ be a random hypothesis set consisting of $h_0$ and independent random vectors in $\{-1,1\}^{u}$ with  i.i.d. uniform random entries. Further let the size of $\hi$ be $N/k$ without counting $h_0$, where 
$N=2\cme^{-2}k\ln \left(k/\delta\right)\exp(8\cmt\gamma^2r_1)$.
    With the above, we have with probability at least $1-\delta/k$ over $\hi$ that:
    \begin{enumerate}
        \item There exists a hypothesis $\shh\in \hi$ such that 
        \begin{align*}
            \sum_{i\in \Ss} D_i \shh(i)\geq 2\gamma
        \end{align*}
        where $\shh=h_0$ if $\sum_{i=1,i\in \Ss}^{u-r_1} D_i > 1/2+\gamma$ else $\shh$ is random.
        \label{lemma:item1}
    \end{enumerate}
    Further, if $~\sum_{i=1,i\in \Ss}^{u-r_1} D_i \leq 1/2+\gamma$ and $r\leq |\overline{\Sss}\cap [u-r_1]|$ 
    \begin{enumerate}[resume]
        \item $\shh$ in \cref{lemma:item1} is such that the first $r$ entries of $\{\shh(i)\}_{{i\in\overline{\Sss}\cap [u-r_1] }}$ has at least $(1/2+\cg\gamma/2) r$ minus signs.\label{lemma:item2}
    \end{enumerate}
\end{restatable}
Recall that $\Ss$ in AdaBoost is just the training set $S$ (without the labels which are all $1$ in our setting).
Intuitively, the first item states that there is a hypothesis with a sufficient advantage on the training set.
In the case that there is not much weight on the first part (where $h_0$ is positive, i.e. $\dist$ focuses on the second part) and there are at least $r$ points in the first part that are not part of the training set, then Item 2 states that we can even find a hypothesis with many minus signs in this first part (outside of the training data).
Since we are trying to learn the all ones hypothesis, those minus signs will induce a large error later on.

Further, we need the following lemma in the proof of \cref{lemmabadweaklearners}, to say that for any linear combination over hypotheses in $\he$ can not achieve a large advantage on too many points within the last $r_1$ points of $\Xs$. Thus, it is impossible to achieve a large advantage where $h_0$ is $-1$.
\begin{restatable}{lemma}{linearcomlemma}
    \label{linearcomlemma}
    Let $\mathbf{A}$ be uniform random in $\{-1,1\}^{r \times n}$ and assume that $r\geq 40\lg(n)$. 
    With probability at least $1-2^{-0.01r}$, it holds for all $w \in \R^n$ with $\|w\|_1=1$ that $\mathbf{A}w$ has at least $r/10$ entries $i$ with $(\mathbf{A}w)_i < 14\sqrt{\lg(n)/r}$.
\end{restatable}

We will prove \cref{linearcomlemma} later in this section.
We now restate \cref{lemmabadweaklearners} and give the proof under the assumption that \cref{lemmabadhypothesis} and \cref{linearcomlemma} hold.

\badweaklearners*

\begin{proof}
Let $\he=\cup_{i=1}^k\hi$ and $\htt=\cup_{i=k+1}^{2k}\hi$ for independent outcomes of $\hi$ from \cref{lemmabadhypothesis}.
In the proof, we consider the three bullets of the lemma separately. 

The first bullet, i.e. the bound on the size of $\sh$ follows immediately from \Cref{lemmabadhypothesis} and the bound on $|\hi|$ of $N/k$, and the fact that we use $2k$ hypothesis sets $\hi$ in $\sh$.
We thus end up with at most 
\begin{align*}
    2N=4\cmo^{-2}k\ln \left(k/\delta\right)\exp(8\cmt\gamma^2r_1)
\end{align*}
random hypothesis in $\sh$ adding $h_0$ gives the desired bound on $\sh$'s size. 
Thus, what remains to be shown is the second and third bullet of \cref{lemmabadhypothesis}.

\paragraph{Second bullet / Properties of the event $\eg$:}
We now show the second bullet, which intuitively states that $\gh$ outputs hypothesès with a $\gamma/4$ advantage on $S$, many minus signs in $\frs$, and linear combinations of them on the last $r_1$ points can not all have large margins (the part where $h_0$ is $-1$).

Let the function $\gh$ that searches for the first hypothesis in $\sh_1,\ldots,\sh_k$ which has a $\gamma/4$ advantage (i.e. fulfils \Cref{item1weaklearn} in \cref{lemmabadweaklearners}) for a given distribution $\dist\in\distc$ and additionally has at least $(1/2+\alpha\gamma/2)r$ minus signs in the first $r$ points of $\overline{\Sss}\cap [u-r_1]=\bar{S}\cap[u-r_1]$, matching \cref{alg:majorityvoter}. 
If there is no such hypothesis, $\gh$  chooses the hypothesis $h_0$.
Let further $S\in \spo$ and define $\ege$ to be the event (over the outcome $\Hyp$ of $\sh$) that 
\begin{align}
    \ege \coloneqq \bigg\{\Hyp: \forall \dist\in \dist_S \,\exists h\in \Hyp \text{ such that: } &\sum_{i\in S} \Dist(i) h(i)\geq \gamma/4 \text{ and } h(\frs) \text{ has } (1/2+\alpha\gamma/2)r \text{ minus signs}\nonumber \\ &\text{ or }  h_0\in \Hyp \text{ and } \sum_{i\in S} \Dist(i) h_0(i)\geq \gamma/4 \bigg\}.
\end{align}
\mrrand{isn't $h_0$ always part of $\mathcal H$?}
$\ege$ will be one part of $\eg$ ($\eg$ will be a union of two events) and used in arguing for \cref{item1weaklearn}, \cref{lemmabadweklarnsitem2}, and \cref{item3weaklearn}.
We now argue that $\ege$ happens with probability at least $1-\delta$ over $\he$.
For this we run \Cref{majorityvoter} on input $S\in\spo$ and $\sh_1,\ldots,\sh_k$. 
Using \Cref{lemmabadweaklearners}, we show that a run of \Cref{majorityvoter} finishes on input $S$ and $\sh_1,\ldots,\sh_k$  with probability at least $1-\delta$ and that this implies that $\he$ is in the event $\ege$. 
To see this we show that whenever \Cref{majorityvoter} finishes, it produces an $f$ such that $f(i)\geq \gamma/4$ for any $i\in S$ (large margin on $S$) and that the hypotheses that $f$ is made of (when they are not $h_0$) have at least $(1/2+\alpha\gamma/2)r$ minus signs in the first $r$ points of $\bar{S}\cap[u-r_1]$. We then notice that $f(i)\geq \gamma/4$ for any $i\in S$ implies that for any $\dist\in \dist_S$ one of the hypotheses $f$ is made of must have a $\gamma/4$ advantage on the all-ones label.
This follows from $\Sss=S$ for $\dist\in\dist_S$, $\dist$ being a probability distribution, $f=(1/k)\sum_{j=1}^kh_j$, and
\begin{align} 
  \gamma/4 \leq \sum_{i\in S}D_S(i)f(i)=\sum_{j=1}^{k} 1/k \sum_{i\in S} D_S(i)h_j(i).\label{gammaadvantages}
\end{align}

We therefore conclude that the event that \cref{majorityvoter} finishes is contained in $\ege$. Thus if we can show that \cref{majorityvoter} with input $S\in\spo$ and $\sh_1,\ldots,\sh_k$ finish with probability at least $1-\delta$, then $\he$ is in $\ege$ with probability at least $1-\delta$ over $\he$. We show that \cref{majorityvoter} finishes with probability $1-\delta$ in the end of this section and has the promised guarantees.

To handle \Cref{item4weaklearn}, we define the event $\egt$ as
\begin{align}
    \egt \coloneqq \braces{\Hyp: \forall w\in \Delta_{\Hyp\backslash h_0} \text{ at least } r_1/10 \ i\text{'s in } \{u-r_1+1,\ldots,u\} \text{ satisfies: }  \sum_{j\in |\Hyp|} w_j h_j(i)\leq 14\sqrt{\log \left(|\Hyp|\right)/r_1}}.
\end{align}
We show that $\he$ is in $\egt$ with probability at least $1-2^{-0.01r_1}$ over $\he$. 
To see this, we form a matrix of all hypotheses created by $\sh_1,\ldots,\sh_k$ excluding $h_0$ (the hypotheses as columns). 
Now using $r_1\geq40\lg(|\he|)$ by the assumption in the bullet of the lemma, \Cref{linearcomlemma} invoked on the lower $ r_1\times|\he| $ part of this matrix, gives us that $\he$ is in $\egt$ with probability at least $1-2^{-0.01r_1}$.
Now setting $\eg=\ege\cap \egt$ and using a union bound we get that that $\he$ is in $\eg$ with probability at least $1-\delta-2^{-0.01r_1}.$

First notice that conditioned on $\eg$, we get by the $\egt$ part of $\eg$ that \cref{item4weaklearn} of the second bullet follows. 
From the definition of $\gh$ choosing a hypothesis with  $\gamma/4$ advantage with at least $(1/2+\alpha\gamma/2)r$ minus signs in $\frs$ or else $h_0$ it follows from the $\ege$ part of $\eg$ that \cref{item1weaklearn} holds and the guarantee about at least $(1/2+\alpha\gamma/2)r$ minus signs in $\frs$ of \cref{lemmabadweklarnsitem2}.\mrrand{long sentence}
Further, the part of \cref{lemmabadweklarnsitem2} claiming that the minus signs in $\frs$ of $\gh$ are uniformly distributed between any permutation in $\{-1,1\}^r$ with at least $(1/2+\alpha\gamma/2)r$ minus signs follows from the hypothesis in $\he\backslash h_0$ being random vectors in $\{-1,1\}^u$ with i.i.d. uniform entries, i.e. all outcomes of $\{-1,1\}^r$ with at least $(1/2+\alpha\gamma/2)r$ minus signs are equally likely.
That the entries of $\he\backslash h_0$ are i.i.d. and the constrains different from, $\gh$ having at least $(1/2+\alpha\gamma/2)r$ minus signs in $\frs$, imposed in $\ege$ and $\egt$ only depend on points in $\overline{\frs}$ gives the claims of independence in \cref{item3weaklearn} for $\gh$ on $\frs$.\mrrand{odd sentence}

What is left to show is that \Cref{majorityvoter}  with input  $\sh_1,\ldots,\sh_k$ and $S$ finishes with probability at least $1-\delta$ and that on the event that \Cref{majorityvoter} finishes it produces an $f$ such that $f(i)\geq \gamma/4$ for any $i\in S$ and that the hypotheses that $f$ is made of (when they are not $h_0$) have at least $(1/2+\alpha\gamma/2)r$ minus signs in the first $r$ points of $\frs=\bar{S}\cap[u-r_1]$.
By \Cref{lemmabadhypothesis}, \Cref{majorityvoter} with $S$ and $\mathbf{H}_1,\ldots,\mathbf{H}_k$ as input finishes with probability at least $(1-\delta/k)^k\geq 1-\delta$, where we have used the independence of the hypothesis sets $\mathbf{H}_1,\ldots,\mathbf{H}_k$.
The claim that the $f$ produced when \cref{majorityvoter} finishes consists of hypotheses (when they are not $h_0$) with at least $(1/2+\alpha\gamma/2)r$ minus signs in $\frs$ follows from \cref{choosehzero}, \cref{chooseminuses}, and \cref{choosefail} of \cref{majorityvoter}.

Thus, we still need to show that  $f(i)\geq \gamma/4$ for all $i\in S$ when \Cref{majorityvoter} finishes. 
In this case, we know  that the hypotheses $h_1,\ldots,h_k$ chosen by \Cref{majorityvoter} fulfill \cref{choosehzero} and \cref{chooseminuses} in \Cref{majorityvoter} which ensures that hypothesis chosen in the $i$'th round $h_i$ for the distribution in the $i$'th round $D_i$ has a $2\gamma$ advantage. 
Let $\eta=\frac{1}{2}\ln\frac{1+2\gamma}{1-2\gamma}$ and $f_k=k\cdot f=\sum_{i=1}^k h_i$.
We now follow a standard AdaBoost argument to show that $\exp \left(-\eta f_k(i)\right)\leq \exp\left(\ln \left( |S| \right)-2k{\gamma}^{2}\right)$, 
for any $i\in S$ when \Cref{majorityvoter} finishes.

Showing $\exp \left(-\eta f_k(i)\right)\leq \exp\left(\ln \left( |S| \right)-2k{\gamma}^{2}\right)$, 
for any $i\in S$ implies that 
$f(i)\geq \left(2k{\gamma}^{2}-\ln \left(|S|\right)\right)/(k\eta)$ 
and since for $\gamma < 1/4$, it holds that 
\begin{align*}
  \eta ~=~ \frac{1}{2} \ln \left(1+\frac{4\gamma}{1-2\gamma}\right) ~\leq~ \frac{2\gamma}{1-2\gamma} ~\leq~ 4\gamma
\end{align*}
we get
\begin{align*}
  f(i)~\geq~ \frac{2k\gamma^2-\ln \left(|S|\right)}{4k\gamma} 
  ~\geq~ \frac{\gamma}{2}-\frac{\ln(|S|)}{4k\gamma}
\end{align*}
and using that $k=\ln(u)\gamma^{-2}$ and $S\subseteq [u]$ it follows that $f(i)\geq \gamma/4$. 
Thus, if we show  
 $\exp \left(-\eta f_k(i)\right)\leq \exp\left(\ln \left( |S| \right)-2k\gamma^2 \right)$
for all $i\in S$ we are done. 
Let $Z_l$ be the normalization factor for the multiplicative weight update step in \cref{alg:majorityvoter}.
We now argue that $\exp \left(-\eta f_j(i)\right)=|S|D_{j+1}(i)\prod_{l\in [j]} Z_l$ for all $j\in[k]$ and $i\in [u]$ and that $\prod_{l\in [k]} Z_l\leq (1-2\gamma^{2})^k$.
Showing these two relations implies that
\begin{align}
    \exp\left(-\eta f_{k}(i)\right)
    ~\leq~ |S|\prod_{l\in [k]} Z_l 
    ~\leq~ |S|(1-2\gamma^{2})^k
    ~\leq~ \exp\left(\ln \left(|S|\right)-2k\gamma^2\right)
    \label{eq:adaboostAim}
 \end{align}
where the first inequality uses $D_{k+1}\leq 1$ and the last inequality follows from $\log(1+x)\leq x$ for $x> -1$.

We show that $\exp \left(-\eta f_j(i)\right)=|S|D_{j+1}(i)\prod_{l\in [j]} Z_l$ for all $j\in[k]$ and $i\in [u]$ by induction. 
For the induction base $j=1$ we have $\exp \left(-\eta f_1(i)\right)=\exp \left(- \eta h_1(i)\right)$ and $|S|D_{2}(i) Z_1=|S|D_{1}(i)\exp \left(-\eta h_1\right)=\exp \left(-\eta h_1\right)$, where we have used that $D_{2}(i)=D_1(i)\exp\left(-\eta h_1(i)\right)/Z_1$ and $D_1(i)=1/|S|$.
For the induction step we have
\begin{align*}
  \exp\left(-\eta f_{j+1}(i)\right) 
  = \exp\left(-\eta \left(f_{j}(i)+h_{j+1}(i)\right)\right)
  = |S|D_{j+1}(i) \prod_{l\in [j]} Z_l \exp \left(-\eta h_{j+1}(i)\right)
  = |S|D_{j+2}(i)\prod_{l\in [j+1]} Z_l
\end{align*}
where the second equality follows from the induction hypothesis for $j$ and the last by 
 $D_{j+2}(i)=D_{j+1}(i)\exp(\eta h_{j+1}(i))/Z_{j+1}$ (see \cref{alg:majorityvoter}).

To show $\prod_{l\in [k]} Z_l\leq (1-2\gamma^{2})^k$, i.e. the second inequality in \cref{eq:adaboostAim}, we show $Z_l\leq (1-2\gamma^{2})$ for $l=1,\ldots,k$. 
Using that $\exp(\eta)=\left(\frac{1+2\gamma}{1-2\gamma}\right)^{1/2}$ we notice that 
\begin{align}\label{AdaBoostineqaulityire-first}
  Z_l &=~\sum_{i\in S}D_l(i)\exp \big({-\eta} h_l\left(i\right)\big)\nonumber\\
  &=\sum_{\substack{i\in S: \\ h_l(i)=1}}D_l(i)\exp \left(-\eta \right)+\sum_{\substack{i\in S: \\ h_l(i)=-1}}D_l(i)\exp \left(\eta \right)\nonumber\\
  &=\sum_{\substack{i\in S: \\ h_l(i)=1}} D_l(i) \sqrt{\frac{1-2\gamma}{1+2\gamma}}+\left(1-\sum_{\substack{i\in S: \\ h_l(i)=1}}D_l(i)\right) \sqrt{\frac{1+2\gamma}{1-2\gamma}} \nonumber\\
  &=\left(\sum_{\substack{i\in S: \\ h_l(i)=1}} D_l(i)\frac{1}{1+2\gamma} + \left(1-\sum_{\substack{i\in S: \\ h_l(i)=1}}D_l(i)\right) \frac{1}{1-2\gamma}\right) \sqrt{\left(1+2\gamma\right)\left(1-2\gamma\right)}.
\end{align}
Using that we noticed that \cref{choosehzero}, \cref{chooseminuses}, and \cref{choosefail} in \cref{majorityvoter} together with \cref{majorityvoter} finishing implied  $\sum_{i\in S} D_j(i)h_j(i)\geq 2\gamma$ for any $j\in k$  we get that
\begin{align*}
  \sum_{\substack{i\in S \\ h_l(i)=1}} D_{l}(i)=\sum_{\substack{i\in S }} D_{l}(i)\frac{1+h_l(i)}{2}\geq 1/2+\gamma,
\end{align*}
and using this together with $\frac{x}{1+2\gamma}+\frac{1-x}{1-2\gamma}$ being decreasing we get
\begin{align*}
  \left(\sum_{\substack{i\in S \\ h_l(i)=1}}D_l(i) \frac{1}{1+2\gamma} + \left(1-\sum_{\substack{i\in S \\ h_l(i)=1}}D_l(i)\right)\frac{1}{1-2\gamma}\right)\leq 1.
\end{align*} 
Further using that $(1-2x)(1+2x)=1-4x^{2}\leq (1-2x^{2})^{2}$ we conclude by \Cref{AdaBoostineqaulityire-first}  that $Z_l\leq (1-2\gamma^{2})$ as claimed.

\paragraph{Third bullet / Properties of $\thh$:}
Let $\thh$ be such that given a $\dist \in\distc$ it returns the first hypothesis in $\htt$ that has a $\gamma/4$ advantage on $\dist$ otherwise report fail. 
Note that $\thh$ does not include any adversarial behavior, it is a simple and straightforward $\gamma$-weak learner.
We now show with probability at least $1-\delta$ over $\htt$ that $\thh$ succeeds simultaneously for all $\dist \in \distc$.
Here, we use a slightly different argument compared to the case for $\gh$ above and run \Cref{majorityvoter} in a slightly modified version.
The slight modification is that in \cref{chooseminuses} we have no constraints on the number of minus signs in the first $r$ positions of $\bar S\cap [u-r_1]$ and that we run the algorithm with the input $\Xs$ and $\sh_{k+1},\ldots,\sh_{2k}$ (instead of $\sh_1,\dots,\sh_k$).
We then show that this variant of \Cref{majorityvoter} succeeds with probability at least $1-\delta$ and that the produced $f$ satisfies $f(i)\geq \gamma/4$ for all $i\in [u]$.
By the same argument as above for \Cref{gammaadvantages}, it follows that $f(i)\geq \gamma/4$ for all $i\in u$ implies that for any $\dist$ there exist an $\shh \in \sh_{k+1},\ldots,\sh_{2k}$ with a $\gamma/4$ advantage on $\dist$.
Thus, the event that this slightly modified version of \Cref{majorityvoter} succeeds on $\Xs$ and $\sh_{k+1},\ldots,\sh_{2k}$ is contained in the event
\begin{align*}
    \braces{\Hyp:\forall \dist \in \distc ~\exists h\in \Hyp \text{ such that: } \sum_{i\in [u]} \dist(i)h(i)\geq \gamma/4}.
\end{align*}
Hence, with probability at least $1-\delta$ for any $\dist\in\distc$, $\thh$ finds a hypothesis in $\htt$ with $\gamma/4$ advantage (choosing the first it finds) and outputs this as the weak learner for the distribution $\dist$. 

The claim that \cref{majorityvoter} with $\Xs$ and $\mathbf{H}_{k+1},\ldots,\mathbf{H}_{2k}$ succeeds with probability at least $1-\delta$ over $\htt$ follows as in the $\gh$-case from \cref{lemmabadhypothesis} and $\mathbf{H}_{k+1},\ldots,\mathbf{H}_{2k}$ being independent. 

We now notice that when we argued that the non-modified version of \cref{majorityvoter} finishing would produce an $f$ such that $f(i)\geq \gamma/4$ for $i\in S$, we never used the constraint on the minus signs, and only that $|S|\leq u$.
Thus, reusing the above arguments but now for the modified version of \cref{majorityvoter} finishing, with $S=\Xs$, again yields that the produced $f$ satisfies $f(i)\geq \gamma/4$ for $i\in \Xs$, which concludes the proof of \cref{lemmabadweaklearners}.

\end{proof}

Having established the proof of \cref{lemmabadweaklearners} using  \cref{linearcomlemma} and \cref{lemmabadhypothesis}  we now move on to the proof of those. 
We start by restating and giving the proof of \cref{linearcomlemma}.

\linearcomlemma*

\begin{proof}
The following proof proceeds by bounding the probability of the complementary event of the above, i.e. we will show that the probability of there existing a $w\in\mathbf{R}^n$, $\|w\|=1$ such that $\mathbf{A}w$ has strictly less than $r/10$ entries such that $(\mathbf{A}w)_i<14\sqrt{\lg(n)/r}$ happens with probability at most $2^{-0.01r}$. 
For this we first discretize the set of all unit vectors, call this set $\Ws$. We then show that if there exists a unit vector with the above property, then there exists a vector $\tilde{w} $ in $\Ws$ such that $\mathbf{A}\tilde{w} $ has at least $(13/20)r$ strictly positive entries. Now using that $\mathbf{A}$ has i.i.d. uniform $\{-1,1\}$-random variables as entries, $(\mathbf{A}\tilde{w})_i$ is strictly positive with a probability at most $1/2$, i.e. in expectation we see at most\mrrand{exactly?} $(1/2)r$ strictly positive entries.
The result then follows by applying Hoeffding's inequality and union bounding over $\Ws$. 

Consider the set $\Ws$ containing all $w$ whose coordinates $w_i$ are of the form $j_i 40 \lg(n)/r$ for integers $j_i \in \{- r/(40\lg n),\dots,r/(40\lg n)\}$ and $\|w\|_1=1$.
\mrinline{why do we need negative weights? This way, $w_i$ is bounded by $1600$ and it is not clear how many steps the quantization has (that depends on $r/\lg n$), is this intentional? Mi: I believe it should be $j_i \in \{-r/(40\lg n),\dots,r/(40\lg n)\}$}
We now want to bound $|\Ws|$. 
For this, consider throwing $r/(40\lg(n))$ balls with a sign and absolute value $40\log(n)/r$ into $n$ buckets. 
There are $(2n)^{r/(40 \lg n)} \leq 2^{r/20}$ outcomes of this experiment. We now map each $w\in \Ws$ to an outcome of the above experiment. For this, notice that $\sum_{i=1}^n j_i=r/(40\lg(n))$ since $w\in \Ws$ has unit length. Now for a $w\in \Ws$ consider any outcome of the experiment where for $i=1,\ldots,n$: $j_i$ balls fell into the $i$'th bucket, and all the balls signs coincide with $\sign(w_i)$.  In this case the value of the $i$'th bucket is the same value as $w_i$. Thus, we conclude that $|\Ws|\leq 2^{r/20}$.
 
Now consider an outcome $A$ of the random matrix $\mathbf{A}$ and assume there exists $w \in \R^n$ with $\|w\|_1=1$ such that $Aw$ has strictly less than $r/10$ entries $i$ with $(Aw)_i < 14\sqrt{\lg(n)/r}$. We now show that this implies that there exists a vector $\tilde{w}\in \Ws$ such that $A\tilde{w} $ has at least $(13/20)r$ strictly positive entries. For $t=1,\dots,r/(40\lg n)$ sample independently an index $\mathbf{j(t)}$ from $w$ such that the $i$'th index is sampled with probability $|w_i|/\|w\|_1$. Let $\mathbf{\tilde{w}}$ be the vector whose $i$'th coordinate is $\mathbf{j_i} \sign(w_i) 40 \lg(n)/r$. Here $\mathbf{j_i}$ denotes the number of times index $i$ was sampled.

Consider any coordinate $(A\mathbf{\tilde{w}})_i$. 
Using i.i.d. random variables $X_t$ taking the value $a_{i,\mathbf{j(t)}}\sign(w_{\mathbf{j(t)}})40 \lg(n)/r$, we can write $(A\mathbf{\tilde{w}})_i$ as $ \sum_{t=1}^{r/(40\lg n)} X_t$.
Note that $\E[X_t]=\sum_{i=1}^na_{i,j}w_i 40\lg(n)/r=(Aw)_i40\lg(n)/r$. 
Thus, we see that $\E[(A\mathbf{\tilde{w}})_i ] = (r/(40\lg n))\,\E[X_1] =(Aw)_i.$ Notice that since $X_t$ takes values in $\{-40\lg(n)/r,40\lg(n)/r\}$, its variance is at most $(40\lg(n)/r)^2$. 
Further, by the independence of the $X_t$'s, we have that $(A\mathbf{\tilde{w}})_i$ has variance at most $(r/(40\lg n)) (40\lg(n)/r)^2 = 40\lg(n)/r$. 
Thus, Chebyshev's inequality implies that $\Pr[\,|(A\mathbf{\tilde{w}})_i-(Aw)_i|> 2\sqrt{40\lg(n)/r}] \leq 1/4$. Now noticing that $\mathbf{\tilde{w}} \in \Ws$ and using the linearity of expectation, we conclude that there must be some vector $\tilde{w} \in \Ws$ for which there are  less than $r/4$ entries $i$ such that $|(A\mathbf{\tilde{w}})_i-(Aw)_i|> 2\sqrt{40\lg(n)/r}$. This, combined with the assumption of $(Aw)_i<14\sqrt{\lg(n)/r}$ for strictly less than $r/10$ entries, implies that $A\tilde{w}$ has at least $r-r/10 -r/4 = (13/20)r$ entries $i$ such that $(Aw)_i\geq14\sqrt{\lg(n)/r}$  and $|(A\tilde{w})_i-(Aw)_i|> 2\sqrt{40\lg(n)/r}$. 
Thus, we conclude that at least $(13/20)r$ entries $i$ satisfy $(A\tilde{w}_i) \geq 14 \sqrt{\lg(n)/r} - 2\sqrt{40 \lg(n)/r}>0$, i.e. if there exists $w \in \R^n$ with $\|w\|_1=1$ such that $\mathbf{A}w$ has strictly less than $r/10$ entries $i$ then there also exists $\tilde{w}\in\Ws$ such that $\mathbf{A}\tilde{w}$ has at least $(13/20)r$ entries that are strictly positive. 

Thus, what remains is to argue that $\Ws$ with small probability over $\mathbf{A}$ contains a vector $w$ with at least $(13/20)r$ entries $i$ such that $(\mathbf{A}w)_i > 0$. For this, consider any fixed $w \in \Ws$. The probability that $(\mathbf{A}w)_i>0$ is at most $1/2$ for all $i$. Now 
Hoeffding's inequality implies that the probability that there are $(13/20)r$ entries $i$ with $(\mathbf{A}w)_i>0$ is no more than $\exp(-2((3/10)r)^2/(4r)) = \exp(-(9/200) r)$. 
A union bound over all of $\Ws$ (recall $|\Ws|\leq 2^{r/20}$) shows that the probability that there exists a vector $w\in\Ws$ which has at least $(13/20)r$ strictly positive entries is at most $e^{-(9/200)r} 2^{r/20} < 2^{-0.01r}$ over $\mathbf{A}$. 
Thus, we conclude that the probability of existence of a $w \in \R^n$ with $\|w\|_1=1$ such that $\mathbf{A}w$ has strictly less than $r/10$ entries $i$ with $(\mathbf{A}w)_i < 14 \sqrt{\lg(n)/r}$ is at most $2^{-0.01r}$ which concludes the proof.
\end{proof}

To show \cref{lemmabadhypothesis} we need the following corollary which follows from a use of the  Montgomery-Smith inequality \cite{10.2307/2048015}.  The corollary says that a linear combination of i.i.d. uniform $\{-1,1\}$-variables where the coefficient's absolute values sums to at least $1/2-\beta/2$ with some probability are greater than $\beta$. This will be used in \cref{lemmabadhypothesis} to say that $\hi$ for a given $\dist\in\distc$  contains a hypothesis $\shh$ with an advantage of $2\gamma$. 

\begin{restatable}{corollary}{montcoral}\label{montcoral}
    There exist universal constants $\ck, \cpt \leq 1$, and $\cptt\geq 1$ such that for $\beta\leq \ck /6$, $x\in\mathbb{R}^ n $, $x_i\geq 0 \ \forall i\in [n]$, and  $\sum_{i=1}^n x_i\geq (1-\beta)/2$, we have for a random $\shh\in\{-1,1\}^n$ with i.i.d. uniform entries that
    \begin{align*}
    \pr\left[\sum_{i=1}^{n}\shh(i)x_i\geq \beta\right]\geq 
    \cpt\exp\left(-\cptt \frac{16  \beta^2 n}{\ck^2 }\right)
\end{align*}
\end{restatable}

We will show \cref{montcoral} after the proof of \cref{lemmabadhypothesis}. We now restate and give the proof of \cref{lemmabadhypothesis}

\lemmabadhypothesis*

\begin{proof}
If the distribution $\dist$ has more than $1/2+\gamma$ mass on the points $1,\ldots,u-r_1$, i.e. $\sum_{i=1,i\in \Ss}^{u-r_1}D_i>1/2+\gamma$, we have $\sum_{i=u-r_1+1,i\in \Ss}^{u}D_i<1/2-\gamma$. 
Thus, we notice that $h_0$ satisfies
\begin{align*}
    \sum_{i\in \Ss} D_i h_0(i)=\sum_{\substack{i=1\\ i\in \Ss}}^{u-r_1}D_i-\sum_{\substack{i=u-r_1+1\\ i\in \Ss}}^{u}D_i\geq 2\gamma,
\end{align*}
i.e. $h_0$ fulfills \Cref{lemma:item1}.

Now assume that $\sum_{i=1,i\in \Ss}^{u-r_1}D_i\leq 1/2+\gamma$. 
Then we have $1/2-\gamma$ mass on the points $\{u-r_1+1,u\}\cap \Ss$, i.e.  $\sum_{i=u-r_1+1,i\in \Ss}^{u}D_i\geq 1/2-\gamma$.
Since we know that the entries of any $\shh$ in $ \sh_i$ for $\shh\not=h_0$ are i.i.d. uniform $\{-1,1\}$-variables, we get that $\sum_{i=1, i\in \Ss}^{u-r_1}\shh(i)D_i\geq 0$ with probability $1/2$. 
Thus, we give a lower bound on the probability of $\sum_{i=u-r_1+1,i\in \Ss}^{u}\shh(i)D_i\geq 2\gamma $. 
Using that $\sum_{i=u-r_1+1,i\in \Ss}^{u}D_i\geq 1/2-\gamma$ (by the assumption in this paragraph), \Cref{montcoral} implies that for $2\gamma \leq \cmo$
\begin{align*}
    \pr\left[\sum_{i=u-r_1+1,i\in \Ss}^{u}D_i\shh(i)\geq 2\gamma \right]\geq \cme\exp(-4\cmt\gamma^2 r_1)
\end{align*}
so we conclude by the independence of the entries in $\shh(i)$ that
\begin{align}\label{lemma:eq1}
    \pr\left[\sum_{i\in \Ss}D_i\shh(i)\geq 2\gamma \right]\geq 
    \pr\left[\sum_{\substack{i=1\\i\in \Ss}}^{u-r_1}D_i\shh(i)\geq 0 ,\sum_{\substack{i=u-r_1+1\\ i\in \Ss}}^{u}D_i\shh(i)\geq 2\gamma \right]
    \geq \cme\exp(-4\cmt\gamma^2 r_1)/2,
\end{align} 
where $\cmo$, $\cme$, and $\cmt$ are universal constants (some of them are the product of universal constants in \cref{montcoral}).
Thus, \Cref{lemma:item1} holds for every $\shh$ in $\sh_i\backslash h_0$ with at least the above probability. 
Now if $r\leq |\overline{\Ss}\cap [u-r_1]|$  let $F_r$ be the first $r$ indices of $\overline{\Ss}\cap \{1,\ldots,u-r_1\}$.
Note that $F_r$ has the same role as $\frs$ in other parts of the paper, but in this lemma we make no assumptions about the support of $\dist$.
Then by \Cref{montcoral}, we get that for $\cg \gamma\leq \cmo$
\begin{align}\label{lemma:eq2}
    \pr\left[\sum_{i\in F_r}\shh(i)/r\leq -\cg \gamma \right]
    =\pr\left[\sum_{i\in F_r}\shh(i)/r\geq \cg \gamma \right]\geq\cme\exp(-\cmt(\cg\gamma)^2r)\geq \cme\exp(-4\cmt\gamma^2 r_1),
\end{align}
where the equality is due to the $\shh(i)$ being i.d.d. uniform $\{-1,1\}$-variables and the last inequality follows from $r\leq r_1$.
If we have $\sum_{i\in F_r}\shh(i)/r\leq -\cg \gamma$ then $\{\shh(i)\}_{i\in F_r}$ must contain at least $(1/2+\cg\gamma/2)r$ minus ones.
Thus, we conclude by \Cref{lemma:eq1} and \Cref{lemma:eq2}, and the independence of the entries of $\shh$ that
\begin{align*}
    \pr\left[\sum_{i\in \Ss}\shh(i)D_i\geq 2\gamma,~|\{i\in F_r\mid \shh(i)=-1\}|\geq (1/2+\cg\gamma/2)r \right]\geq \cme^2\exp(-8\cmt\gamma^2r_1)/2.
\end{align*}
By the definition of $N=2\cme^{-2}k\ln \left(k/\delta\right)\exp(8\cmt\gamma^2r_1)$ we get that we have that 
\begin{align*}
    \cme^2\exp(-8\cmt\gamma^2r_1)/2=\frac{k\ln(k/\delta)}{N}.
\end{align*}
Now define $f(h)=\indicator{}_{\{\sum_{i\in \Ss}\shh(i)D_i\geq 2\gamma,|\{i\in F_r\mid \shh(i)=-1\}|\geq (1/2+\cg\gamma/2)r \}}$. 
Using $f$, independence of the $\shh$'s in $\sh_i$, and that the size of $\hi$ is $N/k$ we get that
\begin{align*}
    \Pr\left[ \exists \shh\in \sh_i \text{ s.t. } f(\shh)=1\right]
    =~ &1-\Pr\left[ \forall \shh\in \sh_i \text{ we have } f(\shh)=0\right]\\
    =~ &1-\Pr\left[ f(\shh)=0\right]^{N/k}\\
    =~ &1-\left(1-\Pr\left[ f(\shh)=1\right]\right)^{N/k}\\
    \geq~ &1-\left(1-\frac{k\ln(k/\delta)}{N}\right)^{N/k}\\
    \geq~ &1- \exp(-\ln \left(k/\delta\right))\\
    =~ &1-\delta/k
\end{align*}
where the last inequality follows from $(1+x/n)^n=\exp\left(n\ln(1+x/n)\right)\leq \exp\left(x\right)$ for $n\geq 1$ and $x\geq -1$, since $\ln(1+x)\leq x$ for $x\geq -1$. 
This shows \Cref{lemma:item1} in the case $~\sum_{i=1,i\in \Ss}^{u-r_1} D_i \leq 1/2+\gamma$ and  \Cref{lemma:item2} if $r\leq |\overline{\Sss}\cap [u-r_1]|$  which finishes the proof of \cref{lemmabadhypothesis}.
\end{proof}

We now prove and restate \cref{montcoral}.

\montcoral*
\begin{proof}
\mrinline{in which cases is this a useful function? Why do we care about K? Mi: I dont think we care particular about it other than it is the quantity Montgomery Smith talks about. We are more interested in it with respect to \cref{mont:eq1}´. In the beginning i didn't state what $K$ was, but i also felt that was weird.
MR: It would really be nice to hear something about $K$, informal or formal, before Eq 29 (instead of an afterthought)
}
\mrinline{We should write that we actually do not care about what K is and make very explicit that Eq 29 and Eq 30 are just taken from [11].}
In the following we will assume that the $x_i$'s  are ordered by their absolute value, which we can assume  without loss of generality  since the $\shh(i)$'s are i.d.d. uniform $\{-1,1\}$-variables.
By \cite{10.2307/2048015} there exist universal constants $\ck$, $\cpt$, and $\cptt$ such that 
\begin{align}\label{mont:eq1}
    f(x,t):=  \sum_{i=1}^{\min\left(\left\lceil t^2\right\rceil, n \right)} x_i+t\sqrt{\sum_{i=\left\lceil t^2\right\rceil+1}^{ n } x_i^2},
\end{align}
and
\begin{align}\label{mont:eq2}
    \pr\left[\sum_{i=1}^{ n }\shh(i)x_i\geq \ck f(x,t)\right]\geq \cpt\exp(-\cptt t^2).
\end{align}

Notice that we may assume that $\ck < 1$.
If $\ck$ was greater than $1$, we could lower it to $1$ and the claim in  \Cref{mont:eq2} would still hold. 
Similarly, we also assume $\cpt\leq 1$ and $\cptt\geq 1$.

Now consider $t = \frac{4\beta \sqrt{n}}{\ck}$  which implies that $t^2\leq  n /2$ since $\beta\leq \ck /6$.
Thus the first sum of \Cref{mont:eq1} goes up to $\lceil t^2\rceil$. 
Formally, if $\ck f(x,t)\geq \ck \sum_{i=1}^{\left\lceil t^2\right\rceil} x_i\geq \beta $ we get by \Cref{mont:eq1} and \Cref{mont:eq2} that
\begin{align*}
    \pr\left[\sum_{i}^{ n }\shh(i)x_i\geq \beta \right]\geq\pr\left[\sum_{i=1}^{ n }\shh(i)x_i\geq \ck f(x,t)\right]
    \geq \cpt \exp(-\cptt t^2)
    = \cpt\exp\left(-\cptt \frac{16  \beta^2 n}{\ck^2}\right).
\end{align*}

For the other case, assume that $\ck  \sum_{i=1}^{\left\lceil t^2\right\rceil} x_i~\leq~ \beta $, which combined with $\sum_{i=1}^n x_i\geq 1/2-\beta/2$ implies that 
\begin{align*}
    \ck  \sum_{i=\left\lceil t^2\right\rceil+1}^{ n } x_i
    =\ck  \left(\sum_{i=1}^{ n } x_i-\sum_{i=1}^{ \left\lceil t^2\right\rceil } x_i\right)
    \geq\ck  (1-\beta-2\beta/\ck )/2.
\end{align*}

By Cauchy-Schwarz (in the second inequality below) and $\lceil t^2\rceil\leq n$ we get that
\begin{align}
    \ck  (1-\beta-2\beta/\ck)/2
    &\leq \ck  \sum_{i=\left\lceil t^2\right\rceil+1}^{ n } 1\cdot x_i
\leq \ck  \sqrt{\,| n -\left\lceil t^2\right\rceil| \sum_{i=\left\lceil t^2\right\rceil+1}^{ n } x_i^2 }
    \leq  \ck   \sqrt{ n  \sum_{i=\left\lceil t^2\right\rceil+1}^{ n } x_i^2.\nonumber }\\
 &\Rightarrow (1-\beta-2\beta/\ck )/2\leq\sqrt{ n  \sum_{i=\left\lceil t^2\right\rceil+1}^{ n } x_i^2 . }\label{mont:eq3}
\end{align}
We notice that $\beta\leq \ck/6$ implies $(1-\beta-2\beta/\ck)\geq 1/2$. 
From \Cref{mont:eq1} we get with \Cref{mont:eq3}, $t=\frac{4\beta \sqrt{n}}{\ck }$, and $(1-\beta-2\beta/\ck)\geq 1/2$ that \mrinline{now I have a problem with the very last root appearing here, without the root, we are missing a factor of $2$ afterwards. Mi: I know it. It should be fixed now by choosing t larger. Just had to figure out what the "easiest" fix was such that we didnt have to change it all the way up the paper. }
\begin{align*}
     \ck f(x,t)
    \geq  \ck  t\sqrt{\sum_{i=\left\lceil t^2\right\rceil+1}^{ n } x_i^2}
    =4\beta \sqrt{ n \sum_{i=\left\lceil t^2\right\rceil+1}^{ n } x_i^2}
    \geq 4\beta \frac{(1-\beta-2\beta/\ck )}{2 }\geq \beta
\end{align*}
Now using this and \Cref{mont:eq2} we get that 
\begin{align*}
    \pr\left[\sum_{i}^{ n }\shh(i)x_i\geq \beta \right]
    \geq\pr\left[\sum_{i=1}^{ n }\shh(i)x_i\geq  \ck f(x,t)\right]
    \geq \cpt \exp \brackets{-\cptt t^2}
    = \cpt\exp\left(-\cptt \frac{16 \beta^2 n}{\ck^2}\right)
\end{align*}
as in the other case which finishes the proof.
\end{proof}

We now have shown \cref{lemmabadweaklearners} and the two lemmas \cref{linearcomlemma} and \cref{lemmabadhypothesis} that are used in the lemma.
This leaves us to prove \cref{biaslemma} and \cref{couponscollectors} which both appear in the proof of the main theorem.
We start by restating \cref{biaslemma}.

\biaslemma*

\begin{proof}
First, if there is $j\in\{1,\ldots,d\}$ such that $w_j\geq \crt\beta$ (i.e. there is a hypothesis $h_j$ with a large weight in the output of \cref{alg:majorityvoter}) we get that
\begin{align*}
    \pr\left[\sum_{i=1}^dw_i\shh(i)\leq -\crt \beta \right]\geq \pr\left[\sum_{\substack{i=1\\i\not=j}}^dw_i\shh(i)\leq 0, w_jr_j \leq -\crt \beta \right]\geq 1/4
\end{align*}
which follows from the $\shh(i)$'s being biased towards minus so if we changed them to i.i.d. uniform $\{-1,1\}$-variables the above probability would be lower and equal to $1/4$.

Thus, we may assume that $\|w\|_{\infty}\leq \crt\beta$, i.e. the largest entry in $w$ is less than $\crt\beta$.
We now introduce the random variables $\eta_i$ and $\tilde{\shh}(i)$ where $\tilde{\shh}(i)$ are i.i.d. uniform $\{-1,1\}$-variables and the $\eta_i$'s  have the distribution $\pr\left[\eta_i=1|\tilde{\shh}(i)=-1\right]=1$, $\pr\left[\eta_i=-1|\tilde{\shh}(i)=1\right]=2\cre\beta$ and $\pr\left[\eta_i=1|\tilde{\shh}(i)=1\right]=1-2\cre\beta$.
We immediately get
\begin{align*}
    \pr\left[\eta_i\tilde{\shh}(i) =-1\right]&=1/2+1/2(2\cre\beta)=1/2+\cre\beta\text{\qquad and }\\
    \pr\left[\eta_i\tilde{\shh}(i) =1\right]&=1/2(1-(2\cre\beta))=1/2-\cre\beta 
\end{align*}
thus $\eta_i\tilde{\shh}(i)$ has the same distribution as $\shh(i)$. 
Using this decomposition of the $\shh(i)$'s we get that 
\begin{align}\label{lemmarade:eq1}
    &\pr\left[\sum_{i=1}^dw_i\shh(i)\leq -\crt \beta \right]\nonumber\\
    =~&\pr\left[\sum_{i=1}^dw_i\eta_i\tilde{\shh}(i)\leq -\crt \beta \right]\nonumber\\
    =~&\pr\left[\sum_{i=1}^dw_i\tilde{\shh}(i)+\sum_{i=1}^dw_i(\eta_i-1)\tilde{\shh}(i)\leq -\crt \beta \right]\nonumber\\
    \geq~& \pr\left[\sum_{i=1}^dw_i\tilde{\shh}(i)\leq 0, \sum_{i=1}^dw_i(\eta_i-1)\tilde{\shh}(i)\leq -\crt \beta \right]\nonumber \\
\geq ~&1-\frac{1}{2}-\pr\left[\sum_{i=1}^dw_i(\eta_i-1)\tilde{\shh}(i)> -\crt \beta \right]
\end{align}
where the last inequality follows from $\pr\left[A\cap B\right]\geq 1-\pr\left[A\right]-\pr\left[B\right]$ and the $1/2$-term by a weighted sum of i.i.d. uniform $\{-1,1\}$-variables being symmetric around 0. 
We now notice that $(\eta_i-1)\tilde{\shh}(i)$ has the same distribution as a random variable  $-2x_i$ where $x_i$ follows $\pr\left[x_i=0\right]=1-\cre\beta$ and $\pr\left[x_i=1\right]=\cre\beta$.
We also see that $\E\left[\sum_{i=1}^n -2w_ix_i\right] =-2\cre\beta$ and by independence of the $x_i$'s 
\begin{align*}
    \V\left(\sum_{i=1}^n -2w_ix_i\right) 
    = 4\sum_{i=1}^{d} w_i^2\left(\E\left[x_i^2\right]-\E\left[x_i\right]^2\right)
    \leq4 \sum_{i=1}^{d} (\crt\beta \,w_i) \left(\cre\beta-\left(\cre\beta\right)^2\right) 
    = 4\cre \crt \beta^2 \left(1-\cre\beta\right)
\end{align*}
where the inequality follows from $\|w\|_\infty\leq \alpha'\beta$ and the last equality uses $\sum_{i=1}^d w_i=1$.
Using that the $\tilde{h}(i)$'s follow the same distribution as $-2x_i$ we get from Chebyshev's inequality, the above calculation of the expected value of $\sum_{i=1}^n -2w_ix_i$, and the upper bounds on its variance that
\begin{align*}
    &\pr\left[\sum_{i=1}^dw_i(\eta_i-1)\tilde{\shh}(i)> -\crt \beta \right]
    = \pr\left[\sum_{i=1}^d w_i(-2x_i) > -\crt \beta \right]
    = \pr\left[\sum_{i=1}^d 2w_i(-x_i+\cre\beta)> (2\cre-\crt) \beta \right]\\
&\qquad\qquad\leq \frac{4\cre\crt\beta^2(1-\cre\beta)}{(2\cre-\crt)^2\beta^2}
    \leq \frac{4\cre\crt}{\left(2\cre-\crt\right)^2}
\end{align*}
where the last inequality uses that $\beta< 1/(2\cre)$.

Thus, we conclude by the above and \Cref{lemmarade:eq1} that in the case that $\|w\|_{\infty}\leq \crt\beta$ we have
\begin{align*}
    \pr\left[\sum_{i=1}^dw_i\shh(i)\leq -\crt \beta \right] \geq \frac{1}{2}-\frac{4\cre\crt}{(2\cre-\crt)^2}.
\end{align*}
Together with the case that $\|w\|_{\infty}\geq \crt\beta$ the claim follows.
\end{proof}

We now restate and prove \cref{couponscollectors}
\couponscollectors*

\begin{proof}
First, notice that seeing a new item in the next sample after having seen $i$ distinct items happens with probability 
\begin{align*}
    p_i=\frac{\cce m/\ln \left(m/r\right)-i}{\cce m/\ln \left(m/r\right)}.
\end{align*}
Now if we use $X_i$ to denote the number of samples between having seen $i$ distinct items and $i+1$ distinct items, we can write $X$ as $\sum_{i=0}^{\cce m/\ln \left(m/r\right)-2r-1} X_i$, i.e. as sum of independent geometric random variables with success probability $p_i$. 
By Theorem 3.1 in \cite{JANSON20181} for $0<\lambda\leq 1$ it holds that
\begin{align}\label{couponineq}
    \pr\left[X\leq \lambda \E\left[X\right]\right]\leq \exp\left(-\min_{i=0,\ldots,\cce m/\ln \left(m/r\right)-2r-1}(p_i)\E\left[X\right](\lambda-1-\ln\left(\lambda\right))\right).
\end{align}
We now notice that 
\begin{align*}
    \min_{i=0,\ldots,\cce m/\ln \left(m/r\right)-2r-1}\left(p_i\right)
    = \frac{2r+1}{\cce m/\ln \left(m/r\right)}
    \geq \frac{2r}{\cce m/\ln \left(m/r\right)}
\end{align*}
and that 
\begin{align}\label{lowerboundcoupons}
    \E\left[X\right]
    =~ &\sum_{i=0}^{\cce m/\ln \left(m/r\right)-2r-1} \frac{\cce m/\ln \left(m/r\right)}{\cce m/\ln \left(m/r\right)-i}\nonumber\\
    =~&\cce m/\ln \left(m/r\right)\sum_{i=2r+1}^{\cce m/\ln \left(m/r\right)} \frac{1}{i}\nonumber\\
    \geq~& \cce m/\ln \left(m/r\right)\int_{2r+1}^{\cce m/\ln \left(m/r\right)} \frac{1}{x} dx\nonumber\\
    =~&\cce m/\ln \left(m/r\right)\ln \left(\frac{\cce m/\ln \left(m/r\right)}{2r+1}\right)\nonumber\\
    \geq~& \cce (m/\ln \left(m/r\right))\ln \left(\frac{\cce m/\ln \left(m/r\right)}{4r}\right)
\end{align}
where the first inequality follows from $1/x$ being monotonically decreasing. 
Using that $x/\log(x)\geq  \sqrt{x}$ for $x\geq 1$ and $\cce\geq 8$ we get that $\E\left[X\right]\geq \cce \left(m/\ln \left(m/r\right)\right)\ln \left(\cce\sqrt{m/r}/4\right) \geq \cce m / 2$.

We can now combine all those ingredients.
By choosing $\lambda = 2/\cce$ and using $\cce \geq 8$ we get that $\lambda -1 -\ln(\lambda) \geq 1/2$. 
First notice that, this choice of $\lambda$ with $\E\left[X\right] \geq \cce m / 2$ implies $\pr[X\leq m] \leq \pr[X \leq \lambda \E[X]]$.
Together with the bound on the minimum of the $p_i$ and the lower bound on $\E[X]$ from \Cref{lowerboundcoupons} we get from \Cref{couponineq} that
 \begin{align*}
    \pr\left[X\leq m\right]\leq\pr\left[X\leq \lambda \E\left[X\right]\right]\leq \exp\left(-\frac{2r\E\left[X\right]\left(\lambda-1-\ln \left(\lambda\right)\right)}{\cce m/\ln \left(m/r\right)}\right)
    \leq \exp\left(-r\ln \left( \frac{\cce m/\ln \left(m/r\right)}{4r}\right)\right)
\end{align*}
From $m \geq 4r$ we get that $(m/r)/\ln(m/r)\geq 1$.
Together with $\cce \geq 8$ we get that $\ln \left( (\cce m/\ln \left(m/r\right))/\left(4r\right) \right)\geq 1$ and since $r\geq 1$ we conclude that $ \pr\left[X\leq m\right]\leq 1/2$ as claimed which concludes the proof.
\end{proof}

 \section{Conclusion}\label{sec:conclusion}

We have presented a lower bound on the sample complexity of AdaBoost, establishing that AdaBoost is sub-optimal by at least one logarithmic factor.
In the proof, we make use of an adversarial weak learner that accumulates errors outside of the training set. 
Technically, this is achieved by relying on concentration and anti-concentration bounds to show that a random hypothesis set will be able to achieve both an advantage within the training set and a negative advantage on a small subset of points outside of it.
In order to work, the weak learner needs to know the training set $S$, which happens to be the case in AdaBoost and many of its variants.
This makes our lower bound applicable to a variety of boosting algorithms, showing that they are all sub-optimal.

In contrast, the optimal weak-to-strong learner from \citet{optimalWeakToStrong} precisely calls the weak learner on subsets of $S$, avoiding the lower bound.
One key question here is whether a generalization of their idea allows to reach optimal generalization performance with a simple majority vote as in AdaBoost instead of their two-level majority scheme.
Another interesting open question is the exact sample complexity of AdaBoost which currently has a logarithmic gap between our lower bound and the best known upper bound.

\section*{Acknowledgements}

Supported by Independent Research Fund Denmark (DFF) Sapere Aude Research Leader grant No 9064-00068B.

\bibliographystyle{icml2023}
\bibliography{references}

\end{document}